\RequirePackage{fix-cm}
\documentclass[smallextended]{svjour3}       
\smartqed  
\usepackage{graphicx}
\usepackage{mathptmx}      
\usepackage{latexsym}
\usepackage{url}
\usepackage{amsmath}
\usepackage{amsfonts}
\usepackage{amssymb}
\usepackage{enumitem}

\def\ar{\leftarrow}

\newcommand{\as}{``}

\newcommand{\be}{\begin{em}}
\newcommand{\ee}{\end{em}}
\newcommand{\bb}{\begin{bf}}
\newcommand{\eb}{\end{bf}}
\newcommand{\I}[1]{\relax\ifmmode\mbox{\it#1}\else{\it#1}\fi}

\newcommand{\tbs}{\hspace*{4mm}}

\newcommand{\tbl}{\hspace*{12mm}}

\newcommand{\no}{not\,}

\newcommand{\meno}{\medskip\noindent}

\newcommand{\rif}{~\ref}
\newcommand{\IF}{\mbox{\,:-\,}}

\newcommand{\emptySeq}{\relax\ifmmode \varepsilon\else$\varepsilon$\fi}   
\newcommand{\tuple}[1]{\langle #1 \rangle}

\newcommand{\D}{\mathit{Dir}}
\newcommand{\R}{\mathit{Reach}}

\def\KB{\mathit{KB}}
\def\ACC{\mathit{ACC}}

\def\act{\mathit{act}}
\def\U{{\cal U}}
\def\T1{T\!+\!1}

 \usepackage{color}

\begin{document}

\title{Multi-Context Systems: Dynamics and Evolution\thanks{This work is partially supported by INdAM-GNCS-17 project, by Xunta de Galicia, Spain (projects GPC ED431B 2016/035 and 2016-2019 ED431G/01 for CITIC center) and by the European Regional Development Fund (ERDF).}
}


\author{Pedro~Cabalar \and Stefania~Costantini \and Giovanni~De~Gasperis \and Andrea~Formisano}


\institute{P. Cabalar \at
              University of Corunna, Spain  \\
              \email{cabalar@udc.es}           
           \and
           S. Costantini \at
              Universit{\`a} di L'Aquila, Italy\\ \email{stefania.costantini@univaq.it}
                         \and
           G. De Gasperis \at
              Universit{\`a} di L'Aquila, Italy\\ \email{giovanni.degasperis@univaq.it}
           \and
           A. Formisano \at
              Universit{\`a} di Perugia, GNCS-INdAM, Italy\\ \email{andrea.formisano@unipg.it}
}

\date{Received: date / Accepted: date}

\maketitle

\begin{abstract}
Multi-Context Systems (MCS) model in Computational Logic distributed systems composed of heterogeneous sources, or
``contexts'', interacting via special rules called ``bridge rules''.
In this paper, we consider how to enhance flexibility and generality in bridge-rules definition and application. In particular, we introduce and discuss some formal extensions of MCSs useful for a practical use in dynamic environments, and we try to provide guidelines for implementations.

\end{abstract}

\section{Introduction}\label{introduction}
Multi-Context Systems (MCSs) have been proposed 
in Artificial Intelligence and Knowledge Representation
to model information exchange among heterogeneous sources \cite{BrewkaE07,BrewkaEF11,BrewkaEP14}.
MCSs do not make any assumption about such sources nor require them to be homogeneous; rather, the MCS approach deals explicitly
with their different representation languages and semantics.
These sources, also called \emph{contexts} or \emph{modules}, interact through 
special interconnection expressions called \emph{bridge rules}; such rules are very similar in syntax and in meaning to logic programming rules as seen in Datalog (cf. \cite{lloyd:foundations,lpneg:survey})
or Answer Set Programming (ASP, cf. \cite{GelLif88,GelLif91,MarTru99,Gelfond07} for foundations. See also~\cite{ErdemGL16} and the references therein for applications of ASP), save that atoms in their bodies refer to knowledge items to be obtained from external contexts.
So, a literal $(c_j:p)$ (where $p$ is an atom) in the body of a bridge rule occurring in context $c_i$ means that $p$ can be proved
in context $c_j$ (given $c_j$'s underlying logic and semantics and given $c_j$'s knowledge base present contents),
vice versa $\no (c_j:p)$ means that $p$ cannot be proved in $c_j$.
In case the entire body \as succeeds'' then the bridge rule can be applied, and its conclusion can be exploited within context~$c_i$.

The reason why MCSs are particularly interesting is that they aim at a formal modeling 
of real applications requiring access to distributed sources, possibly on the web.
In many application domains the adoption of MCSs can bring real advances, whenever different types of heterogeneous information
systems are involved and a rigorous formalization should be adopted, also in view of reliability and verifiability of the resulting systems.
Notice that this kind of systems often involves agents;
MCSs encompassing logical agents have in fact been proposed in the literature (cf.,~\cite{CostantiniDM1}).

Given the potential impact of MCSs for practical knowledge representation and reasoning,
there are some aspects in which their definition is still too abstract.
In this paper, we introduce and discuss some formal extensions of MCSs useful for a practical use in dynamic environments and we try to provide guidelines for implementations.
The paper considers the following aspects:
\renewcommand{\labelenumi}{(\roman{enumi})}
\begin{enumerate}
\item
The specification of bridge-rule grounding and of \emph{grounded equilibria}, and a concept of \emph{proactive} bridge-rule activation, i.e., a bridge-rule in our approach is
no longer reactively applied whenever applicable, but rather it must be explicitly enabled by the context where it occurs.
\item
The explicit definition of the \emph{evolution} of an MCS over time in consequences of updates that
can be more general than sensor input as treated in \cite{BrewkaEP14}.
\item
A further refinement of bridge-rule grounding in an evolving MCS.
\item
Bridge Rules Patterns, by which we make bridge rules parametric w.r.t.\ contexts by introducing special
terms called \emph{context designators} to be replaced by actual context names;
this is in our opinion a very important extension, which allow general bridge
rules schemata to be specialized to actual bridge rules by choosing which specific contexts
(of a certain kind) should specifically be queried in the situation at hand.
\item 
Dynamic MCSs whose composition can change over time as the context may join or leave the system, or can be unavailable at certain time intervals (e.g., due to system/network faults).
\item
Multi-Source option, i.e., whenever a context can obtain the required information from several other contexts, a choice may be possible concerning the \emph{preferred} source, where preferences may change over time and according to current circumstances.
\item
The definition of practical modalities of execution of bridge rules in a distributed MCS.
\end{enumerate}
\renewcommand{\labelenumi}{\alph{enumi}}
All these extensions represent substantial and much-needed improvements to the basic MCS framework, as certified by existing connections with related work on MCSs that treat (though in a different way) the same issues. For instance, concerning the fact that a static set of bridge rules associated to a context can sometimes be unsatisfactory, we mention the work of \cite{GoncalvesKL14a}, where, at run-time, 
new bridge rules can be \as observed'', i.e., learned from outside. Such rules can be added to existing ones or can replace them and
 a check on consistency among bridge-rules conclusions is done after the update. This very interesting work is complementary to our
proposal, where new bridge rules are generated inside a context from general meta-level schemata; a combination of the two approaches might confer much more flexibility to contexts' operation.
Also, concerning MCSs that evolve in time, much work has been done, e.g., in \cite{BrewkaEP14,GoncalvesKL14b,GoncalvesKL14a}; our approach either generalizes such proposals or, at least, can be combined with them.

This paper is a revised extended version of the work presented in \cite{CostantiniDG16,CostantiniCF17}.
In the rest of the paper, 
we first recall the MCS approach. Then, after illustrating a motivating scenario, we introduce a list of aspects concerning MCS that are in our opinion problematic, or that might be improved; in the subsequent sections, for every such aspect we discuss and formalize some variations, enhancements, and extensions to basic MCSs that we deem to be necessary/useful to improve the basic paradigm. Subsequently we present a case study, we discuss the complexity of the enhanced framework and finally we conclude. We have chosen to discuss related work whenever a discussion is useful and pertinent rather than in a separate section.

\section{Bridge Rules and Multi-Context Systems}  
\label{MCS}

Heterogeneous MCSs have been introduced in~\cite{GiunchigliaS94}
in order to integrate different inference systems without resorting to non-classical logic systems.
%
Later, the idea has been further developed and generalized to non-monotonic reasoning domains
---see, for instance,  \cite{BrewkaE07,BrewkaEF11,BrewkaEFW11,BrewkaEP14,Serafini2004} among many.
(Managed) MCSs are designed for the construction of systems that need to access multiple (heterogeneous) data sources called \emph{contexts}.
The information flow is modeled via \emph{bridge rules}, whose form is similar to Datalog rules with negation
where however each element in their \as body'' explicitly includes the indication of the contexts from which each item of information
is to be obtained.
To represent the heterogeneity of sources, each context is supposed to be based on its own \emph{logic}, defined in a very general way~\cite{BrewkaEF11}.
A logic $L$ defines its own syntax as a set $F$ of possible \emph{formulas} (or $\mathit{KB}$-elements)
under some signature possibly containing \emph{predicates}, \emph{constants}, and \emph{functions}.
As usual, formulas are expressions built upon the idea of \emph{atom}, that is, the application of a predicate to a number $n$ (the predicate arity) of \emph{terms}.
The latter, in their turn, can be variables, constants, or compound terms using function symbols, as usual.
A \emph{term/atom/formula} is \emph{ground} if there are no variables occurring therein. 
A logic is \emph{relational} if in its signature the set of function symbols is empty, so its terms are variables and constants only. 
Formulas can be grouped to form some \emph{knowledge base}, $kb \in 2^F$.
The set of all knowledge bases for $L$ is defined as some $KB \subseteq 2^F$.
The logic also defines \emph{beliefs} or \emph{data} (usually, ground facts) that can be derived as consequences from a given $kb \in \KB$.
The set $Cn$ represents all possible belief sets in logic $L$.
Finally, the logic specification must also define some kind of inference or entailment.
This is done by defining which belief sets are \emph{acceptable} consequences of a given $kb \in \KB$
with a relation $\ACC \subseteq \KB \times Cn$. Thus, $\ACC(kb,S)$ means that belief set $S$ is an acceptable consequence of knowledge base $kb$.
We can also use $\ACC$ as a function $\ACC: \KB \to 2^{Cn}$ where $S \in \ACC(kb)$ is the same as $\ACC(kb,S)$ as a relation.
To sum up, logic $L$ can be completely specified by the triple $\tuple{\KB, Cn, \ACC}$.

A \emph{multi-context system} (MCS) \(M = \{C_1,\ldots,C_{\ell}\}\) is a
set of ${\ell}=|M|$ \emph{contexts}, each of them of the form $C_i = \tuple{c_i,L_i,kb_i,br_i}$,
where $c_i$ is the context \emph{name} (unique for each context; if a specific name is omitted, $i$ can act as such),  $L_i=\tuple{\KB_i,Cn_i,\ACC_i}$ is a logic,
$kb_i \in \KB_i$ is a knowledge base, and $br_i$ is a set of \emph{bridge rules}. Each bridge rule $\rho \in br_i$ has the form
\begin{eqnarray}
s \ar A_1, \ldots, A_h,\no A_{h+1}, \ldots,\no A_{m} \label{f:br}
\end{eqnarray}
where the left-hand side $s$ is called the \emph{head}, denoted as $hd(\rho)$,
the right-hand side is called the \emph{body}, denoted as $\mathit{body}(\rho)$,
and the comma stands for conjunction.
We define the \emph{positive} (resp. \emph{negative}) \emph{body} as $\mathit{body}^+(\rho)=\{A_1,\dots,A_h\}$ (resp. $\mathit{body}^-(\rho)=\{A_{h+1},\dots,A_m\}$).
The head $hd(\rho)=s$ is a formula in $L_i$ such that $(kb_i \cup \{s\}) \in \KB_i$.
Each element $A_k$ in the body has the form $(c_j:p)$ for a given $j$, $1 \leq j \leq |M|$,
and can be seen as a \emph{query} to the context $C_j \in M$ (possibly different from $C_i$) whose name is $c_j$,
to check the status of belief $p$ (a formula from $L_j$) with respect to the current belief state (defined below) in $C_j$. 
When the query is made in the context $j=i$ we will omit the context name, simply writing $p$ instead of $(c_i:p)$.
A \emph{relational} MCS \cite{FinkGW11} is a variant where all the involved logics are relational and aggregate operators 
in database style (like \emph{sum}, \emph{count}, \emph{max}, \emph{min}) are 
admitted in bridge-rule bodies.
%

A \emph{belief state} (or \emph{data state}) $\vec{S}$ of an MCS~$M$
is a tuple $\vec{S} = (S_1,\ldots, S_{\ell})$ such that for $1 \leq i \leq {\ell}=|M|$, $S_i \in Cn_{i}$.
Thus, a data state associates to each context $C_i$ a possible set of consequences $S_i$.
A bridge rule $\rho \in br_i$ of context $C_i \in M$ is \emph{applicable} in belief state $\vec{S}$ when,
for any $(c_j:p_j) \in body^+(\rho)$, $p_j \in S_j$, and for any $(c_k:p_k) \in body^-(\rho)$, $p_k \not\in S_k$;
so, $\mathit{app}(\vec{S})$ is the set composed of the head of those bridge rules which are \emph{applicable} in~$\vec{S}$.

In managed MCSs (mMCSs)\footnote{For the sake of simplicity,
we define mMCS simply over logics instead of \as logic suite'' as done in~\cite{BrewkaEFW11},
	where one can select the desired semantics among a set of possibilities.} the conclusion $s$, which represents the \as bare'' bridge-rule result,
becomes $o(s)$, where $o$ is a special operator.
The meaning is that $s$ is processed by operator $o$, that can perform
any elaboration, such as format conversion, belief revision, etc.
More precisely, for a given logic $L$ with formulas $F = \{s \in kb\, |\, kb \in \KB\}$, a \emph{management base} is a set of operation names
(briefly, operations) $OP$, defining elaborations that can be 
performed on formulas. 
For a logic $L$ and a management base
$OP$, the set of
operational statements that can be built from $OP$ and $F$ is \(F^{OP} = \{o(s)\, |\, o \in OP, s \in F\}\).
The semantics of such statements is given by a \emph{management function}, \(\mathit{mng}: 2^{F^{OP}}\times \KB \rightarrow \KB\),
which maps a set of operational statements and a knowledge base into a (possibly different)
modified knowledge base.\footnote{We assume a management function to be deterministic, i.e., to produce a unique new knowledge base.}
Now, each context $C_i = \tuple{c_i,L_i,kb_i,br_i,OP_i,mng_i}$ in an mMCS is extended to
include its own management function $mng_i$ which is crucial for knowledge incorporation from external sources.
Notice that, management functions are not required to be monotonic operators.

The application of the management function $mng_i$ to the result of the applicable rules must be acceptable with respect to $ACC_i$.
We say that a belief state $\vec{S}$ is an \emph{equilibrium} for an mMCS $M$ iff, for $1 \leq i \leq |M|$,
\begin{eqnarray}
S_i \in ACC_i(mng_i(\mathit{app}(\vec{S}),kb_i)) \label{f:eq}
\end{eqnarray}

\noindent
Namely, one
\renewcommand{\labelenumi}{(\roman{enumi})}
\begin{enumerate}
\item
 applies all the bridge rules of $C_i$ that are applicable in the belief state~$\vec{S}$ (namely, the set $\mathit{app}(\vec{S})$);
\item
 applies the management function which, by incorporating bridge-rule results into $C_i$'s knowledge base $kb_i$, computes a new knowledge base $kb_i'$;
\item
 determines via $ACC_i$ the set of acceptable sets of consequences of $kb_i'$.
\end{enumerate}
\renewcommand{\labelenumi}{\alph{enumi}}
In an equilibrium, such a set includes $S_i$, i.e., an equilibrium is \as stable'' w.r.t.\ bridge-rule application.

Conditions for existence of equilibria and the complexity of deciding equilibrium existence for mMCS have been studied \cite{BrewkaE07}; roughly speaking,
such complexity depends on the complexity of computing formula (\ref{f:eq}) for each $C_i \in M$.
Algorithms for computing equilibria have been recently proposed \cite{BrewkaEF11,EiterDFK2014}
but, in practice, they are only applicable when open-access to contexts' contents is granted.
For practical purposes, one will 
often provisionally assume that equilibria exist, and that they do not 
include inconsistent data sets.
It has been proved that such \emph{local consistency} is achieved whenever all management functions
are (lc-)preserving, i.e., if they always determine a $kb'$ which is consistent. 

\section{Motivating Scenarios and Discussion}\label{motdisc}

Some of the reasons of our interest in (m)MCSs and bridge-rules stem from a project where
we are among the proponents \cite{FeK2016}, concerning smart Cyber Physical Systems, with
particular attention (though without restriction) to applications in the e-Health field.
The general scenario of such \as F\&K'' (\as Friendly-and-Kind'') systems \cite{FeK2016} is depicted in Figure\rif{FKfig}.

In such setting we have a set of computational entities, of knowledge bases, and of sensors,
all immersed in the \as Fog'' of the Internet of Everything. 
All components can, in time, join or leave the system. 
Some computational components will be agents. In the envisaged
e-Health application for instance, an agent will be in charge of
each patient. The System's engine will keep track of present system's
configuration, and will enable the various classes of rules to work properly. Terminological rules will allow for more flexible knowledge exchange 
via Ontologies. Pattern Rules will have the role of defining and checking coherence/correctness of system's behavior. 
Bridge rules are the vital
element, as they allow knowledge to flow among components in a clearly-specified principled way: referring to Figure\rif{FKfig}, devices for bridge-rule functioning can be considered as a part of the System's engine. Therefore,
F\&Ks are "knowledge-intensive" systems, providing flexible access to dynamic, heterogeneous, and
distributed sources of knowledge and reasoning, within a highly dynamic computational environment.
We basically consider such systems to be
(enhanced) mMCSs: as mentioned in fact, suitable extensions to 
include agents and sensors in such systems already exist. 

\begin{figure}[tb]
{\centering\includegraphics[width=0.95\linewidth]{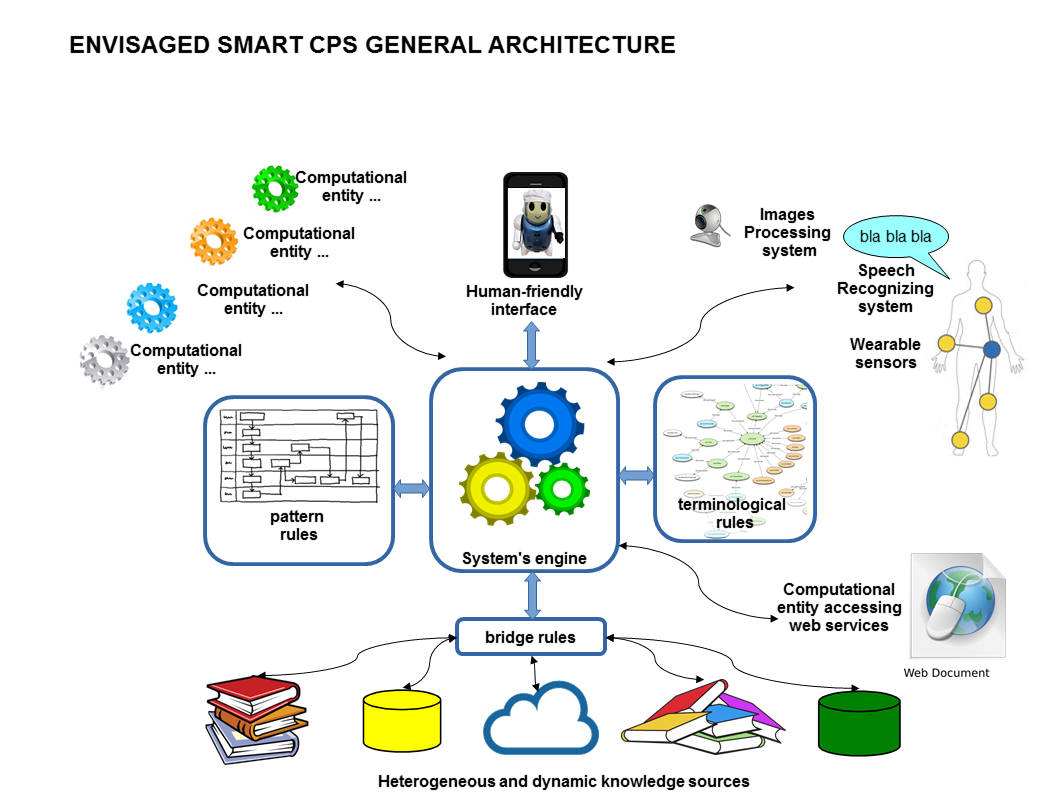}
	\caption{Motivating Scenario}
	\label{FKfig}
}\end{figure}

Another application (depicted in Figure\rif{DINVfig}) in a very different domain though with some analogous features
is aimed at Digital Forensics (DF) and Digital Investigations (DI). 
Digital Forensics is a branch of criminalistics which deals with the identification,
acquisition, preservation, analysis, and presentation of the information content of
computer systems, or in general of digital devices, {in relation to crimes that have been perpetrated}.
The objective of the Evidence Analysis stage of Digital Forensics, or more generally of Digital Investigations, is to identify, categorize, and formalize digital sources of {evidence} (or, however, sources of {evidence} represented in digital form). The objective is to organize such sources of proof into evidences, so as to make them robust in view
of their discussion in court, either in civil or penal trials.
\begin{figure}[tb]
{\centerline{\includegraphics[width=0.95\linewidth]{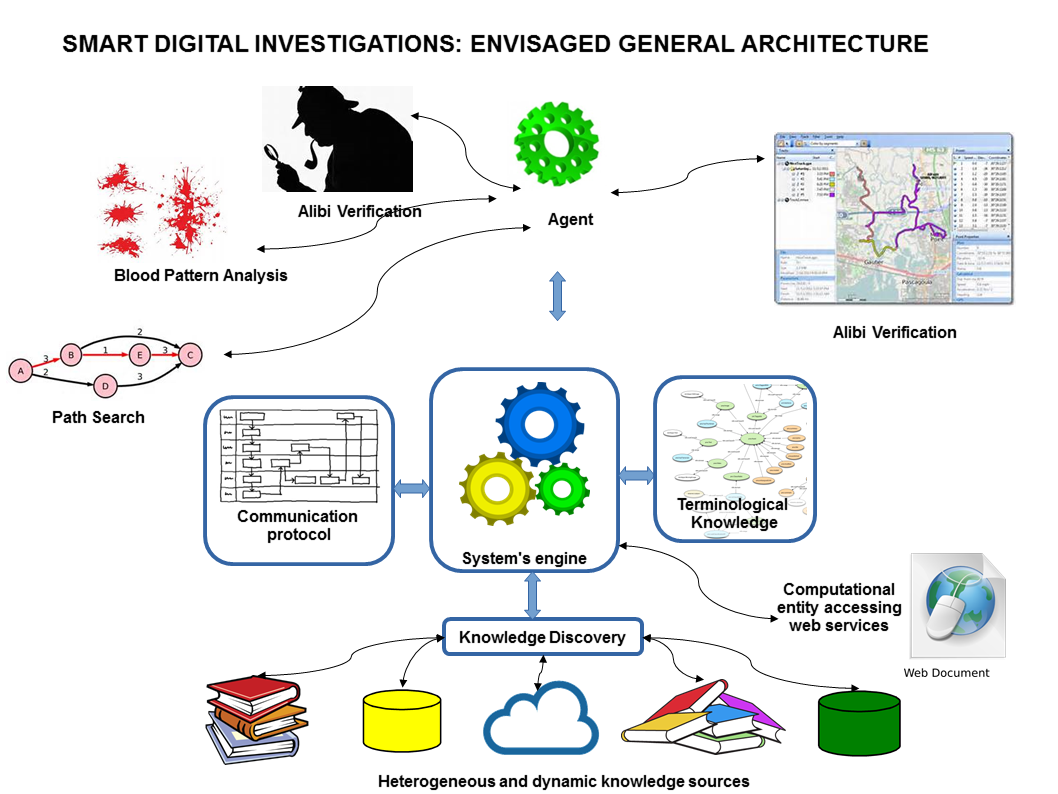}}
	\caption{Digital Forensics (DF) and Digital Investigations (DI) Scenario}
	\label{DINVfig}
}\end{figure}

In recent work\footnote{Supported by Action COST CA17124 \as DIGFORASP: DIGital FORensics: evidence Analysis via intelligent Systems and Practices, start September 10, 2018.''} \cite{CostantiniLPNMR215,OlivieriTh16} we have identified a setting where an intelligent agent is in charge of supporting the human investigator in such activities.
The agent should {help identify, retrieve, and} gather the various kinds of {potentially useful} evidence, process them via suitable reasoning modules, and integrate the results into coherent evidence. In this task, the agent may need to retrieve and exploit knowledge bases concerning, e.g., legislation, past cases, suspect's criminal history, and so on. In the picture, the agent considers: results from blood-pattern analysis on the crime scene, which lead to model such a scene via a graph, where suitable graph reasoning may reconstruct the possible patterns of action of the murderer; alibi verification in the sense of a check of the GPS positions of suspects, so as to ascertain the possibility of her/him being present on the crime scene on the crime time; alibi verification in the sense of double-checking the suspect's declarations with digital data such as computer logs, videos from video cameras situated on the suspect's path, etc. All the above can be integrated with further evidence such as the results of DNA analysis and others. The system can also include Complex Event Processing so as to infer from significant clues the possibility that a crime is being or will be perpetrated.

In our view also this system can be seen as an
(enhanced) mMCSs.
{In reality however, many of the involved data must be retrieved, elaborated, or checked from knowledge bases belonging to organizations which are external to the agent, and have their own rules and regulations for data access and data elaboration. Therefore, suitable modalities for data retrieval and integration must be established in the agent to correctly access such organizations.} Therefore, a relevant issue that we mention but we do not treat here is exactly that of modalities of access to contexts included in an MCS, which can possibly include time limitations and the payment of fees.

In the perspective of the definition and implementation of such kind of systems as mMCSs, the definition recalled in Section\rif{MCS} must be somehow enhanced. Our first experiments show in fact that such definition is, though neat, too abstract when confronted with practical implementation issues.

We list below and discuss some aspects that constitute limitations to the actual flexible applicability of the MCS paradigm, or that might anyway be usefully improved. 
For the sake of illustration, we consider throughout the rest of the paper two running examples.
The first example has been introduced in \cite{CostantiniF16a}, where we have proposed Agent Computational Environments (ACE), that are actually MCSs including agents among their components. In \cite{CostantiniF16a}, an ACE has been used to model a personal assistant agent aiding a prospective college student; the student has to understand to which universities she could send an application, given the subjects of interest, the available budget for applications and then for tuition, and other preferences. Then she has to realize where to perform the tests, and where to enroll among the universities that have accepted the application. The second example concerns F\&Ks, and is developed in detail in Section\rif{casestudy} where we propose a full though small case study. There, a patient can be in need of a medical doctor, possibly a specialized one, and may not know in advance which one to consult.  In the example we adopt a Prolog-like syntax (for surveys about logic programming and Prolog, cf. \cite{lloyd:foundations,lpneg:survey}), with constant, predicate and function symbols indicated by names starting with a lowercase letter, and variable symbols indicated by names starting with an uppercase letter.

\begin{description}
\item[{\it Grounded Knowledge Assumption.}]
Bridge rules are by definition ground, i.e., they do not contain variables.
In \cite{BrewkaEFW11} it is literally stated that [in their examples] they 
\as \emph{use for readability and succinctness schematic
	bridge rules with variables (upper case letters and '\_' [the 'anonymous' variable]) which
	range over associated sets of constants; they stand for all
	respective instances (obtainable by value substitution)}''.
The basic definition of mMCS does not require either contexts' knowledge bases or bridge rules to be finite sets.
Though contexts' knowledge bases will in practice be finite, they cannot be assumed to necessarily
admit a finite grounding, and thus a finite number of bridge-rules' ground instances.
This assumption can be reasonable, e.g., for 
standard relational databases and logic programming under the answer set semantics~\cite{Gelfond07}.
In other kinds of logics, for instance simply \as plain'' general logic programs, it is no longer realistic. 
In practical applications however, there should either be a finite number of applicable (ground instances of)
bridge-rules, or some suitable device for run-time dynamic bridge-rule instantiation and application should be provided.
The issue of bridge-rule grounding has been discussed in \cite{BarilaroFRT13} for relational MCSs,
where however the grounding is performed over a carefully defined finite domain, composed of constants only.
Consider for instance a patient looking for a cardiologist from a medical directory, represented as a context, say, e.g., called $\mathit{med\mbox{-}dir}$.
A ground bridge rule might look like:

\smallskip\noindent
\(\begin{array}{l}
\mathit{find\_cardiologist(maggie\mbox{-}smith)} \ar\\
\tbl\mathit{med\mbox{-}dir : cardiologist(maggie\mbox{-}smith)}
\end{array}\)

\noindent where the patient would confirm what she already knows, i.e., the rule verifies that Dr. Maggie Smith is actually listed in the directory. Instead, a patient may in general intend to use a rule such as:

\smallskip\noindent
\(\begin{array}{l}
\mathit{find\_cardiologist(N)} \ar\\
\tbl\mathit{med\mbox{-}dir : cardiologist(N)}
\end{array}\)

\noindent where the query to the medical directory will return in variable $N$ the name (e.g., Maggie Smith) of a previously unknown cardiologist. Let us assume that $N$ is instantiated to $\mathit{maggie\mbox{-}smith}$ which is the name of a context representing the cardiologist's personal assistant.

Similarly, a prospective student's personal assistant will query universities about the courses that they propose, where this information is new and will be obtained via a bridge rule which is not ground.

Notice that for actual enrichment of a context's knowledge one must allow \emph{value invention}, that is, a constant returned via application of a non-ground bridge rule may not previously occur in the destination module; in this way, a module can  \as learn'' really new information through inter-context exchange.

\item[{\it Update Problem.}]
Considering inputs from sensor networks as done in \cite{BrewkaEP14} is a starting point to make MCSs evolve in time and to have contexts which update their knowledge base and thus cope flexibly with a changing environment.
However, sources can be updated in many ways via the interaction with their environment.
For instance, agents are supposed to continuously modify themselves via sensing and communication with other agents,
but even a plain relational database can be modified by its users/administrators.
Referring to the examples, a medical directory will be updated periodically, and universities will occasionally change the set of courses that they offer, update the fees and other features of interest to students. Where one might adopt fictitious sensors (as suggested in relevant literature) in order to simulate many kinds of update, a more general update mechanism seems in order. Such mechanism should assume that each context has its own update operators and its own modalities of application. An MCS can be \as parametric'' w.r.t.\ contexts' update operators as it is parametric w.r.t.\ contexts' logics and management functions.

\item[{\it Logical Omniscience Assumption and Bridge Rules Application Mechanisms.}]
\ \ In\\MCS, bridge rules are supposed to be applied whenever their body is entailed by the current data state. In fact bridge rules are, in the original MCS formulation, a reactive device where each bridge rule is applied whenever applicable. In reality, contexts are in general not logical omniscient and will hardly compute their full set of consequences beforehand. So, the set of bridge rules which are applicable at each stage is not fully known. Thus, practical bridge rule application will presumably consist in an \emph{attempt of application}, performed by posing queries, corresponding to the elements of the body of the bridge rule, to other contexts which are situated somewhere in the nodes of a distributed system. The queried contexts will often determine the required answer upon request. 
Each source will need time to compute and deliver the required result
and might even never be able to do so, in case of reasoning with limited resources
or of network failures.
Moreover, contexts may want to apply bridge rules only if and when their results are needed. So, a generalization of bridge-rule applicability in order to make it \emph{proactive} rather than reactive can indeed be useful. For instance, a patient might look for a cardiologist (by enabling the bridge rule seen above) only if some health condition makes it necessary.

\item[{\it Static Set of Bridge Rules.}]
Equipping a context with a static set of bridge rules can be a limitation; in fact, in such bridge rules all contexts to be queried are fully known in advance. In contexts' practical reasoning instead, it can become known only at \as run-time'' which are the specific contexts to be queried in the situations that practically arise. To enhance flexibility in this sense, we introduce
\emph{Bridge Rules Patterns} to make bridge rules parametric w.r.t.\ the contexts to be queried; such patterns are meta-level rule schemata where in place of contexts' names we introduce special
terms called \emph{context designators}. Bridge rule patterns can be specialized to actual bridge rules by choosing which specific contexts
(of a certain kind) should specifically be queried in the situation at hand.
This is, in our opinion, a very important extension which avoid a designer's omniscience about how the system will evolve. For instance, after acquiring as seen above, the reference to a reliable cardiologist (e.g., $\mathit{maggie\mbox{-}smith}$), the patient (say Mary) can get in contact with the cardiologist, disclose her health condition $C$, and thus make an appointment for time $T$. This, in our approach, can be made by taking a general bridge rule

\noindent
\(\begin{array}{l}
\mathit{make\_appointment(mary,T)} \ar\\
\tbl \mathit{condition(mary,C)},\\
\tbl \mathit{mycardiologist(c) : consultation\_needed(mary,C,T)}
\end{array}\)

where $\mathit{mycardiologist(c)}$ is intended as a placeholder, that we call \emph{context designator}, that is intended to denote a context not yet known. Such \emph{bridge rule pattern} can be instantiated, at run-time,  to the specific bridge rule

\smallskip\noindent
\(\begin{array}{l}
\mathit{make\_appointment(mary,T)} \ar\\
\tbl \mathit{condition(mary,C)},\\
\tbl \mathit{maggie\mbox{-}smith : consultation\_needed(mary,C,T)}
\end{array}\)

\item[{\it Static System Assumption.}]
The definition of mMCS might realistically be extended in order to encompass settings where the set of contexts
changes over time. This to take into account dynamic aspects such as momentarily disconnections of contexts 
or the fact that  components may freely either join or abandon the system or that inter-context reachability might change over time.

\item[{\it Full System Knowledge Assumption and Unique Source Assumption.}]
A context might know the \emph{role} of another context it wants to query
(e.g., a medical directory, or diagnostic knowledge base) but not its \as name'' that could be, for instance, its URI or anyway
some kind of reference that allows for actually posing a query.
In the body of bridge rules, each literal mentions however a specific context
(even for bridge rule patterns, context designators must be instantiated to specific context names). We enrich the MCS definition with a directory facility, where instead of context names bridge rules may employ queries to the directory to obtain contexts with the required role. Each \as destination'' context might have its own \emph{preferences} among contexts with the same role. We also introduce a structure in the system, by allowing to specify reachability, i.e., which context can query which one. This corresponds to the nature of many applications: a context can sometimes be not allowed to access all or some of the others, either for a matter of security/privacy or for a matter of convenience. Often, information must be obtained via suitable mediation, while access to every information source is not necessarily either allowed or desirable.

\item[{\it Uniform Knowledge Representation Format Assumption.}]
Different contexts might represent similar concepts in different ways:
this aspect is taken into account in \cite{CostantiniDM2} (and so it is not treated here), where ontological
definitions can be exchanged among contexts and a possible global ontology is also considered. 

\item[{\it Equilibria Computation and Consistency Check Assumption.}]
Algorithms for computing equilibria are practically applicable only if open access to contexts' contents is granted.
The same holds for local and global consistency checking.
However, the potential of MCSs is, in our view, that of modeling real distributed systems where contexts
in general keep their knowledge bases private.
Therefore, in practice, one will often just assume the existence of consistent equilibria. This problem is not treated here, but deserves due attention for devising interesting sufficient conditions.
\end{description}

In the next sections we discuss each aspect, and we introduce our proposals of improvement. We devised the proposed extensions in the perspective
of the application of mMCSs; in view of such an application we realized in fact that, primarily, issues related to the concrete modalities of bridge-rule instantiation, activation and execution need to be considered.

\section{Grounded Knowledge Assumption}
\label{grounded}

To the best of our knowledge, the problem of loosening the constraint
of bridge-rules groundedness has not been so far extensively treated in the literature and no satisfactory solution exists.
The issue has been discussed in \cite{BarilaroFRT13} for relational MCSs,
where however the grounding of bridge rules is performed over a carefully defined finite domain composed of constants only.
Instead, we intend to consider any, even infinite, domain.

In the rest of this section we first provide by examples an intuitive idea of how grounding might be performed, step by step, along with the computation of equilibria. Then, we provide a formalization of our ideas.

\subsection{Intuition}

The procedure for computing equilibria that we propose for the case of non-ground bridge rules is, informally, the following.
\renewcommand{\labelenumi}{(\roman{enumi})}
\begin{enumerate}
\item
We consider a standard initial data state $\vec{S}_0$.
\item
We instantiate bridge rules over ground terms occurring in~$\vec{S}_0$; we thus obtain an initial bridge-rule grounding relative to $\vec{S}_0$.
\item
We evaluate whether $\vec{S}_0$ is an equilibrium, i.e., if
$\vec{S}_0$ coincides with the data state~$\vec{S}_1$ resulting from applicable bridge rules.
\item
In case $\vec{S}_0$ is not an equilibrium, bridge rules can now be grounded w.r.t.\ terms occurring in $\vec{S}_1$, 
and so on, until either an equilibrium is reached, or no more applicable bridge rules are generated. 
\end{enumerate}
\renewcommand{\labelenumi}{\alph{enumi}}

It is reasonable to set the initial data state $\vec{S}_0$ from where to start the procedure as the \as basic'' data state where each element consists of the set of ground atoms obtained from
the initial knowledge base of each context, i.e., the context's Herbrand Universe \cite{lloyd:foundations} obtained by substituting variables with constants: one takes functions, predicate symbols and constants occurring in each context's knowledge base and builds the context's data state element by constructing all possible ground atoms.
By definition, a ground instance of a context $C_i$'s knowledge base is in fact an element of $Cn_i$,
i.e., it is indeed a set of possible consequences, though in general it is not acceptable.
As seen below starting from $\vec{S}_0$, even in cases when it is a finite data state, does not guarantee neither the existence of a finite
equilibrium, nor that an equilibrium can be reached in a finite number of steps.

Consider as a first example an MCS composed of two contexts $C_1$ and $C_2$, both based upon plain Prolog-like logic programming
and concerning the representation of natural numbers. Assume such contexts to be
characterized respectively by the following knowledge bases and bridge rules (where $C_1$ has no bridge rule).

\noindent\(
\begin{array}{l}
\%kb_1\\
\tbs nat(0).\\
\%kb_2\\
\tbs nat(succ(X)) \ar nat(X).\\
\%br_2\\
\tbs nat(X)\ar (c1:nat(X)).
\end{array}
\)\smallskip

The basic data state is \(\vec{S}_0 = (\{nat(0)\}, \emptyset)\), where in fact $C_2$'s initial data state is empty because there is no constant occurring in $kb_2$.
The unique equilibrium is reached in one step from via the application of $br_2$ which
\as communicates'' fact $nat(0)$ to $C_2$. In fact, due to 
the recursive rule, we have the equilibrium $(S_1,S_2)$ where\
\(S_1= \{nat(0)\}\) and \(S_2= \{nat(0),nat(succ(0)),nat(succ(succ(0))),\ldots\} \rangle\) 
I.e.,
$S_2$ is an infinite set representing all natural numbers.

If we assume to add a third context $C_3$ with empty knowledge base and a bridge rule $br_3$
defined as~ \(nat(X)\ar (c2:nat(X))\), \, then the equilibrium would be $(S_1,S_2,S_3)$ with $S_3=S_2$.
There, in fact, $br_3$ would be grounded on the infinite domain of the terms occurring in $S_2$, thus
admitting an infinite number of instances. 

The next example is a variation of the former one where a context $C_1$ \as produces'' the even natural numbers
(starting from $0$) and a context $C_2$ the odd ones.
In this case $kb_2$ is empty.

\noindent\(
\begin{array}{l}
\%kb_1\\
\tbs nat(0).\\
\%br_1\\
\tbs nat(succ(X))\ar (c2:nat(X)).\\
\%br_2\\
\tbs nat(succ(X))\ar (c1:nat(X)).
\end{array}
\)\smallskip

We may notice that the contexts in the above example enlarge their knowledge by means of mutual \as cooperation'', similarly to logical agents.
Let us consider again, according to our proposed method, the basic data state \(\vec{S}_0 = (\{nat(0)\}, \emptyset)\).
As stated above, bridge rules are grounded  on the terms occurring therein. 
$\vec{S}_0$ is not an equilibrium for the given MCS.
In fact, the bridge rule in $br_2$, once grounded on the constant $0$, is applicable but not applied.
The data set resulting from the application, namely,
\(\vec{S}' = (\{nat(0)\}, \{nat(succ(0))\})\)
is not an equilibrium either, because now the bridge rule in $br_1$ (grounded on $succ(0)$) is in turn applicable but not applied.

We may go on, as \(\vec{S}'' = (\{nat(0),nat(succ(succ(0)))\},\{nat(succ(0))\})\)
leaves the bridge rule in $br_2$ to be applied (grounded on $succ(succ(0))$), 
and so on. There is clearly a unique equilibrium that cannot however be reached within finite time,
though at each step we have a data state composed of finite sets.

The unique equilibrium (reached after a denumerably infinite number of steps),
is composed of two infinite sets, the former one representing the even natural
numbers (including zero) and the latter representing the odd natural numbers.
The equilibrium may be represented as:

\noindent\(
\vec{E} = \big(\{nat(0),nat(succ^k(0)), \mbox{\ $k$ mod 2\ } = 0\},~ \{nat(succ^k(0)), \mbox{\ $k$ mod 2\ } = 1\}\big)
\)

We have actually devised and applied an adaptation to non-ground bridge rules of the operational characterization introduced in \cite{BrewkaE07} for the grounded equilibrium
of a \emph{definite} MCS. In fact, according to the conditions stated therein $C_1$ and $C_2$ are monotonic and admit at each step a unique set of consequences and
bridge-rule application is not unfounded (cyclic).
In our more general setting  however, the set of ground bridge rules
associated to given knowledge bases cannot be computed beforehand and the step-by-step computation 
must take contexts interactions into account.

Since reaching equilibria finitely may have advantages in practical cases,
we show below a suitable reformulation of the above example, that sets via a practical expedient a bound on the number of steps. The equilibrium reached will be partial, in the sense of representing a subset of natural numbers, but can be reached finitely and is composed of finite sets.

We require a minor modification in bridge-rule syntax that we then take as given in the rest of the paper;
we assume in particular that whenever in some element the body of a bridge rule 
the context is omitted, i.e., we have just $p_j$ instead of $(c_j : p_j)$,
then we assume that $p_j$ is proved locally from the present context's knowledge base.
Previous example can be reformulated as follows, where we assume the customary Prolog syntax and procedural
semantics, where elements in the body of a rule are proved/executed in left-to-right order.
The knowledge bases and bridge rules now are:

\noindent\(
\begin{array}{l}
\%kb_1\\
\tbs nat(0).\\
\tbs count(0).\\
\tbs threshold(t).\\
\%br_1\\
\tbs new(nat(succ(X)))\IF count(C), threshold(T), C < T, (c2:nat(X)).\\
\end{array}
\)

\noindent\(
\begin{array}{l}
\%kb_2\\
\tbs count(0).\\
\tbs threshold(t).\\
\%br_2\\
\tbs new(nat(succ(X)))\IF count(C), threshold(T), C < T,(c1:nat(X)).\\
\end{array}
\)\smallskip

In the new definition there is a counter (initialized to zero) and some threshold, say $t$.
We will exploit a management function
that suitably defines the operator $new$ which is now applied to bridge-rule results.
A logic programming definition of such management function might be the following, where the counter is
incremented and the new natural number asserted. Notice that such definition
is by no means not logical, as we can shift to
the \as evolving logic programming'' extension \cite{jelia02:evolp}.

\noindent\(
\begin{array}{l}
\tbs new(nat(Z))\IF assert(nat(Z)),\, increment(C).\\
\tbs increment(C)\IF retract(count(C)),\, C1\ is\ C+1,\, assert(count(C1)).
\end{array}
\)\smallskip

Consequently, bridge rules will now produce a result only until the counter reaches the threshold, which guarantees the existence of
a finite equilibrium.

\subsection{Formalization}

Below we formalize the procedure that we have empirically illustrated via the examples, so as to generalize 
to mMCS with non-ground bridge rules the operational characterization of \cite{BrewkaE07} for monotonic MCSs
(i.e., those where each context's knowledge base admits a single set of consequences,
which grows monotonically when information is added to the context's knowledge base).
Following~\cite{BrewkaE07}, for simplicity we assume each bridge-rule body to include only positive literals, and the formula $s$ in its head $o(s)$ to be an atom. 
So, we will be able to introduce the definition of \emph{grounded equilibrium of grade $\kappa$}.
Preliminarily, in order to admit non-ground bridge rules we have to specify how we obtain their ground instances and how to establish applicability.

\begin{definition}
\label{groundbr}
Let $r \in br_i$ be a non-ground bridge rule occurring in context $C_i$ of a given mMCS $M$ with belief state $\vec{S}$. 
A ground instance $\rho$ of $r$ w.r.t.~$\vec{S}$ is obtained by substituting every variable 
occurring in $r$ (i.e., occurring either in the elements $(c_j: p_j)$
in the body of $r$ or in its head $o(s)$ or in both) via (ground) terms occurring in~$\vec{S}$.
\end{definition}

For an mMCS $M$, a data state $\vec{S}$ and a ground bridge rule $\rho$, let $\mathit{app^{{\models}_g}}(\rho,\vec{S})$ be a Boolean-valued function 
which checks, in the ground case, bridge-rule body entailment w.r.t.~$\vec{S}$.
Let us redefine bridge-rule applicability.

\begin{definition}
	\label{newapp}
	The set $\mathit{app}(S)$ relative to ground bridge rules which are applicable in a 
	data state $\vec{S}$ of a given mMCS $M = (C_1,\ldots,C_{\ell})$ is now defined as follows.
	
	\noindent\(\begin{array}{l}
	\mathit{app}(\vec{S}) = \big\{hd(\rho)\ |\ \rho \mbox{\ is a ground instance w.r.t.~$\vec{S}$ of some}\\
	\tbl\ \ \ \ \mbox{\,bridge rule\ } r \in br_i, 1 \leq i \leq {\ell},\ 
	 \mbox{\ and $\mathit{app^{{\models}_g}}(\rho,\vec{S})$ = \emph{true}} \big\}
	\end{array}\)
\end{definition}

We assume, analogously to \cite{BrewkaE07}, that a given mMCS is \emph{monotonic},
which here means that for each $C_i$ the following properties hold:
\renewcommand{\labelenumi}{(\roman{enumi})}
\begin{enumerate}
\item
$ACC_i$ is monotonic w.r.t.\ additions to the context's knowledge base, and
\item
$mng_i$ is monotonic, i.e., it allows to only add formulas to $C_i$'s knowledge base. 
\end{enumerate}
\renewcommand{\labelenumi}{\alph{enumi}}
Let, for a context $C_i$, the function $ACC_i'$ be a variation of $ACC_i$ which selects one single set $E_i$
among those generated by $ACC_i$ and let $ACC_i'$ be monotonic. 
Namely,  given a context $C_i$ and a knowledge base $\hat{kb} \in KB_{L_i}$,
$ACC_i'(\hat{kb}) = E_i$ where $E_i \in ACC_i(\hat{kb})$.
Let, moreover, $\infty$ be the first infinite ordinal number isomorphic to the natural numbers.
We introduce the following definition:

\begin{definition}\label{computedatastate}
Let $M= (C_1,\ldots,C_{\ell})$ be an mMCS with no negative literals in bridge-rule bodies, and assume
	arbitrary choice of function $ACC_i'$ for each composing context~$C_i$.
Let, for each $1 \leq i\leq {\ell}$, $gr(kb_i)$ be the grounding of $kb_i$ w.r.t. the constants occurring in any $kb_j$, for $1 \leq j \leq {\ell}$.
	A \emph{data state of grade $\kappa$} is obtained as follows.
	
	\begin{description}
	\item\meno 
	For $i \leq {\ell}$ and $\alpha = 0$, we let $kb_i^0 = gr(kb_i)$, and we let
	\(\vec{S}^{\alpha} = \vec{S}^0 = (kb_1^0,\ldots,kb_{\ell}^0)\)
	
	\item\meno
	For each $\alpha > 0$,
	we let \(\vec{S}^{\alpha} = (S^{\alpha}_1,\ldots,S^{\alpha}_{\ell})\) with
	$S^{\alpha}_i = ACC_i'(kb_i^{\alpha})$
	and where, for finite $\kappa$ and $\alpha \geq 0$ we have
	
	\noindent\tbl\(
	kb_i^{\alpha+1} = \left\{\begin{array}{l}
	 mng_i(\mathit{app}(\vec{S}^{\alpha}),kb_i^{\alpha})\mbox{\ if $\alpha$ $<$ $\kappa$,}\\
	 kb_i^{\alpha} \mbox{\ otherwise}
	\end{array}\right.
	\)
	
	\medskip\noindent
	while if $\kappa = \infty$ we put \(kb_i^{^{\infty}} = \bigcup_{\alpha \geq 0} kb_i^{\alpha}\)
	
		\end{description}
	
\end{definition}

Differently from \cite{BrewkaE07}, 
the computation of a new data state element is provided here according to mMCSs,
and thus involves the application of the management function to the present knowledge base so as to obtain a new one.
Such data state element is then the unique set of consequences
of the new knowledge base, as computed by the $ACC_i'$ function.

The result can be an equilibrium only if the specified grade is sufficient to account
for all potential bridge-rules applications. In the terminology of \cite{BrewkaE07}
it would then be a \emph{grounded equilibrium}, as it is computed iteratively and deterministically from the contexts' initial knowledge bases. We have the following.

\begin{definition}
	Let $M = (C_1, \ldots ,C_{\ell})$ be a monotonic mMCS with no negative literals in bridge-rule bodies. 
	A belief state $\vec{S} = (S_1, \ldots, S_{\ell})$ is a grounded equilibrium of grade $\kappa$
	of $M$ iff $ACC'_i(mng_i(\mathit{app}(\vec{S}),kb^{\kappa}_i) = {S_i}$, for $1 \leq i \leq {\ell}$.	
\end{definition}

\begin{proposition}
Let $M = (C_1, \ldots ,C_{\ell})$ be a monotonic mMCS with no negative literals in bridge-rule bodies. 
A belief state $\vec{S} = (S_1, \ldots, S_{\ell})$ is a grounded equilibrium of grade $\infty$ (or simply 'a grounded equilibrium') for $M$ iff $\vec{S}$ is an equilibrium for $M$.
\end{proposition}

\begin{proof}
$\vec{S}$ is an equilibrium for $M$ because it fulfills the definition by construction. In fact, step-by-step all applicable bridge rules will have been applied, so $\vec{S}$, obtained via $\infty$ steps, is stable w.r.t. bridge rule application. \qed
\end{proof}

Notice that reaching a grounded equilibrium of grade $\infty$ does not always require an infinite number of steps: the procedure of Definition\rif{computedatastate} can possibly reach a fixpoint in a number $\delta$ of steps, where either $\delta = \infty$ or $\delta$ is finite.
In fact, the required grade for obtaining an equilibrium would be $\kappa = \infty$ in the former version of the example,
where in the latter version if setting threshold $t$ we would have $\kappa = t$.

Several grounded equilibria may exist, depending upon the choice of $ACC_i'$. We can state the following relationship with \cite{BrewkaE07}:

\begin{proposition}
	Let $M = (C_1, \ldots ,C_{\ell})$ be a definite MCS (in the sense of \cite{BrewkaE07}), and let
	$\vec{S} = (S_1, \ldots, S_{\ell})$ be a grounded equilibrium for $M$ according to their definition. Then, $\vec{S}$ is a grounded equilibrium of the mMCS $M$'obtained by including in $M$ the same contexts as in $M$, and, for each context $C_i$, letting $ACC_i'=ACC_i$, and associating to $C_i$ a management function~$mng_i$ 
	that just adds to $kb_i$ every $s$ such that $o(s) \in \mathit{app}(\vec{S})$.
\end{proposition}

\begin{proof}
As all the bridge rules in both $M$ and $M'$ are ground, the procedure of Definition\rif{computedatastate} and the procedure described on page~4 (below Definition~11) in~\cite{BrewkaE07} become identical, as we added only two aspects, i.e., considering non-ground bridge rules, and considering a management function, which are not applicable to $M'$. \qed \end{proof}

In \cite{BrewkaE07}, where the authors consider \emph{ground} bridge rules only, they are able to transfer the concept of grounded equilibrium of grade~$\infty$ of a monotonic MCS $M$ to its extensions, where an extension is defined below. Intuitively, an extension $M'$ of (m)MCS $M$ has the same number of contexts, each context in $M'$ has the same knowledge base of the corresponding context in $M$, but, possibly, a wider set of bridge rules; some of these bridge rules, however, \emph{extend} those in $M$ as they agree on the positive bodies.

\begin{definition}
	Let $M = (C_1, \ldots ,C_{\ell})$ be a monotonic (m)MCS with no negative literals in bridge-rule bodies. A monotonic (m)MCS $M' = (C'_1, \ldots ,C'_{\ell})$ is an \emph{extension} of $M$ iff for each $C_i = \tuple{c_i,L_i,kb_i,br_i,OP_i,mng_i}$, we have that $C'_i = \tuple{c_i,L_i,kb_i,br'_i,OP_i,mng_i}$ where $\rho \in br_i$ implies $\exists\rho' \in br_i'$ where $body^+(\rho) = body^+(\rho')$ and $body^-(\rho')$ may be nonempty. We call $\rho$ and $\rho'$ \emph{corresponding} bridge rules.
\end{definition}

In \cite{BrewkaE07} it is stated that a grounded equilibrium $\vec{S} = (S_1, \ldots, S_{\ell})$ of grade~$\infty$ of a definite MCS $M$ is a grounded equilibrium of grade~$\infty$ of any extension $M'$ of $M$ that can be obtained from $M'$ by: canceling all bridge rules where $\vec{S}$ does not imply the negative body, and canceling the negative body of all remaining bridge rules. We may notice that \cite{BrewkaE07} proceeds from $M'$ to $M$ via a reduction similar to the Gelfond-Lifschitz reduction \cite{GelLif88}.
In our case, we assume to extend the procedure of Definition\rif{computedatastate} by dropping the assumption of the absence of negation in bridge-rule bodies. So, the procedure will now be applicable to monotonic MCSs in general.
Then, we have that: 

\begin{proposition}
	Let $M = (C_1, \ldots ,C_{\ell})$ be a monotonic (m)MCS with no negative literals in bridge-rule bodies, and let $M'$ be an extension of $M$. Let
	$\vec{S} = (S_1, \ldots, S_{\ell})$ be a grounded equilibrium for $M$,
	reachable in $\delta$ steps. $\vec{S}$ is a grounded equilibrium for $M'$ iff, if applying the procedure of Definition\rif{computedatastate} to both $M$ and $M'$, we have that $\forall \alpha > 0$, $\rho \in \mathit{app}(\vec{S}^{\alpha})$ iff $\rho' \in \mathit{app}(\vec{S'}^{\alpha})$, where $\rho$ and $\rho'$ are corresponding bridge rules $\vec{S}$ and $\vec{S'}$ are the data states at step $\alpha$ of $M$ and $M'$, respectively.
\end{proposition}

\begin{proof}
The result follows immediately from the fact that at each step corresponding bridge rules are applied in $M$ and $M'$ on the same knowledge bases, so at step $\infty$ the same belief state will have been computed. A proof by induction is thus straightforward. \qed
\end{proof}

\section{Update Problem: Update Operators and Timed Equilibria}
\label{update}

Bridge rules as defined in mMCSs are basically a \emph{reactive} device, as a bridge rule is applied whenever applicable.
In dynamic environments, however,
a bridge rule in general will not be
applied only once, and it does not hold that an equilibrium, once reached, lasts forever. 
In fact, in recent extensions to mMCS, contexts are able to incorporate new data items coming from observations,
so that a bridge rule can be, in principle, \emph{re-evaluated} upon new
observations.
For this purpose, two main approaches have been proposed for this type of mMCSs:
so-called \emph{reactive}~\cite{BrewkaEP14} and \emph{evolving}~\cite{GoncalvesKL14b} multi-context systems.
A \emph{reactive} MCS (rMCS) is an mMCS where the system is supposed to be equipped with a set of sensors that can provide observations from a set $Obs$.
Bridge rules can refer now to a new type of atoms of the form $o@s$, being $o \in Obs$ an observation that can be provided by sensor $s$.
A \as run'' of such system starts from an equilibrium and consists in a sequence of equilibria
induced by a sequence of sets of observations.
In an \emph{evolving} MCS (eMCS), instead, there are special \emph{observation contexts}
where the observations made over time may cause changes in each context knowledge base.
As for the representation of time, different solutions have been proposed. 
For instance, \cite{BrewkaEP14} defines a special context whose belief sets on a run specify a time sequence.
Other possibilities assume an \as objective'' time provided by a particular sensor, or a special
head predicate in bridge rules whose unique role is adding timed information to a context.

However, in the general case of dynamic environments,
contexts can be realistically supposed to be able to
incorporate new data items in several ways, including interaction with
a user and with the environment.
We thus intend to propose extensions to explicitly take into account not only observations 
but, more generally, the interaction 
of contexts with an external environment. Such an interaction needs not to be limited to bridge rules,
but can more generally occur as part of the context's reasoning/inference process.
We do not claim that the interaction with an external environment \emph{cannot} somehow be expressed
via the existing approaches to MCSs which evolve in time \cite{BrewkaEP14,GoncalvesKL14b,GoncalvesKL14a};
however, it cannot be expressed \emph{directly}; in view of practical applications, we introduce explicit and natural devices to cope with this important issue.

We proceed next to define a new extension called \emph{timed} MCS (tmMCS).
In a tmMCS, the main new feature is the idea of \emph{(update) action}.
For each context $C_i$, we define a set $\act_i$ of \emph{elementary actions}, 
where each element $\pi \in \act_i$ is the name of some action or operation that can be performed to update context $C_i$.
We allow a subset $Obs_i \subseteq act_i$ for observing sensor inputs as in~\cite{BrewkaEP14}.
A \emph{compound action} (\emph{action}, for short) $\Pi_i$ is a set $\Pi_i \subseteq \act_i$ of elementary actions that can be simultaneously applied for context $C_i$.
The application of an action to a knowledge base is ruled by an \emph{update function}:
$$\mathit{{\cal U}_i}: KB_i \times 2^{\act_i}  \rightarrow 2^{KB_i} \setminus \{\emptyset\},$$
so that $kb'_i \in \U_i(kb_i,\Pi_i)$ means that the ``new'' knowledge base $kb'_i$ is a possible result of updating $kb_i$ with action $\Pi_i$ under function $\U_i$.
We thus assume ${\cal U}_i$ to encompass all possible updates performed to a module
(including, for instance, sensor inputs, updates performed by a user,
consequences of messages from other agents or changes determined by the context's self re-organization).
Note that $\U_i(kb_i,\Pi_i)$ is (possibly) non-deterministic, but never empty (some resulting $kb'_i$ is always specified).
In some cases, we may even leave the knowledge base unaltered, since it is admitted that $kb_i \in \U_i(kb_i,\Pi_i)$.
Notice, moreover, that update operations can be non-monotonic.

In order to allow arbitrary sequences of updates, we assume a linear time represented by a sequence of \emph{time points} $\{t_T\}_{T\geq 0}$ indexed by natural numbers $T \in \mathbb{N}$ and representing discrete states or instants in which the system is updated.
The elapsed time between each pair of states is some real quantity $\delta=t_{\T1}-t_T>0$.
Given $T \in \mathbb{N}$ and context $C_i$, we write $kb_i[T]$ and $\Pi_i[T]$ to respectively stand for the knowledge base content at instant $T$
 and the action performed to update $C_i$ from instant $T$ to $\T1$.
We define the \emph{actions vector} at time $T$ as {\protect${\boldmath\Pi}$}$[T]=(\Pi_1[T],\dots,\Pi_{\ell}[T])$. 
Finally, $S_i[T]$ denotes the set of beliefs for context $C_i$ at instant $T$ whereas,
accordingly, $\vec{S}[T]$~denotes the belief state $(S_1[T],\dots,S_{\ell}[T])$.
$\vec{S}[T]$ can be indicated simply as $\vec{S}$ whenever $T$ does not matter.

\begin{definition}[timed context]\label{evolvcontexts}
Let $C_i = \tuple{c_i,L_i,kb_i,br_i,OP_i,mng_i}$ be a context in an mMCS.
The corresponding \emph{timed} context w.r.t.\ an initial belief state $\vec{S}$ is
defined as:
\begin{itemize}[ topsep=1pt, itemsep=0pt]
\item $C_i[0] = \tuple{c_i,L_i,kb_i[0],br_i,OP_i,mng_i,\act_i,{\cal U}_i}$
\item $C_i[\T1] = \tuple{c_i,L_i,kb_i[\T1],br_i,OP_i,mng_i,\act_i,{\cal U}_i}$, for $T \geq 0$
\end{itemize}
where
$kb_i[0] := kb_i$ and $\vec{S}[0]:=\vec{S}$, whereas
$kb_i[\T1]{:=}mng_i(\mathit{app}(\vec{S}[T]),kb')$ and $kb' \in {\cal U}_{i}(kb_{i}[T],\Pi_i[T])$.
\end{definition}

\begin{definition}
	We let a tmMCS at time $T$ be $M[T] = \{C_1[T],\ldots,C_{\ell}[T]\}$.
\end{definition}

The initial timed belief state $\vec{S}[0]$ can possibly be an equilibrium, according to original mMCS definition.
Later on, however, the transition from a timed belief state to the next one, and consequently the definition of an equilibrium,
is determined both by the update operators and by the application of bridge rules. Therefore:
\begin{definition}[timed equilibrium]
	A timed belief state of tmMCS $M$ at time $\T1$ is a \emph{timed	equilibrium} iff, for $1 \leq i \leq {\ell}$ it holds that~
	{\(S_i[\T1] \in ACC_i(mng_i(\mathit{app}(\vec{S}[\T1]),kb_i[\T1])\).} 
\end{definition}

The meaning is that a timed equilibrium is now a data state which encompasses bridge rules applicability
on the updated contexts' knowledge bases. 
As seen in Definition\rif{evolvcontexts}, applicability is checked on belief state $\vec{S}[T]$, but bridge rules are applied (and their results incorporated by the management function) on the knowledge base resulting from the update.
The enhancement w.r.t.~\cite{BrewkaEP14} is that the
management function keeps its original role concerning bridge rules, while the update operator
copes with updates, however and wherever performed. 
So, we relax the limitation that each rule involving an update should be
a bridge rule, and that updates should consist of (the combination and elaboration of) simple atoms occurring in bridge bodies.
Our approach, indeed, allows update operators to consider and incorporate any piece of knowledge; for instance, an update can encompass table creation and removal in a context which is a relational database or addition of machine-learning results in a context which collects and processes big data.
In addition, we make time explicit thus showing the timed evolution of contexts and equilibria,
while in \cite{BrewkaEP14} such an evolution is implicit in the notion of system run.

A very important recent proposal to introducing time in mMCS is that
of Streaming MCS (sMCS) \cite{EiterSMCS17}.
The sMCS approach equips MCS with data streams
and models the time needed to transfer data among contexts, and
computation time at contexts. The aim is to model the asynchronous behavior
that may arises in realistic MCS applications. In sMCS, bridge rules can employ
\emph{window atoms} to obtain 'snapshots' of input streams from other contexts.
Semantically, \as feedback equilibria'' extend the notion of MCS equilibria
to the asynchronous setting, allowing local stability in
system \as runs'' to be enforced,
overcoming potentially infinite loops.
To define window atoms the authors exploit the LARS framework \cite{Beck2015},
which allows to define \as window functions'' to specify
\emph{timebased windows (of size k)} out of a data stream $S$.
A \emph{Window atom} $\alpha$ are expressed on a plain atom $A$ and
can specify that $A$ is true at some
time instant or sometimes/always within a time window.
Window atoms may appear in bridge rules: a literal $c_i:\alpha$ in the body
of some bridge rule, where $\alpha$ is a window atom involving atom $A$,
means that the \as destination''
context (the context to which the bridge rule belongs) is enabled to \emph{observe}
$A$ in context $c_i$ during given time window.
This models in a natural way the access to sensors, i.e., cases where $c_i$
is a fictitious context representing a sensor, whose outcomes are by
definition observable. However, window atoms may concern other kinds
of contexts, which are available to grant observability of their belief state.
The sMCS approach also consider a very important aspect, namely that knowledge
exchange between context is not immediate (as in the basic, rather ideal,
MCS definition) but instead takes a time, that is supposed to be bounded,
where bounds associated to the different contexts are known in advance.
So, the approach considers the time that a context will take to evaluate
bridge rules and to compute the management function.
Thus, `runs' of an sMCS will consider that each context's knowledge base
will actually be updated according to these time estimations.

The sMCS proposal is very relevant and orthogonal to ours. We in fact intend as future work to devise an integration of sMCS with our approach. In the next section we consider aspects not presently covered in sMCS, i.e: how to activate bridge rule evaluation only when needed, and how to cope with the case where the time taken by knowledge interchange between contexts is not known in advance.

\section{Static Set of Bridge Rules: Bridge-Rules Patterns}\label{RulePatterns}

In the original MCS definition, each context is equipped with a static set of bridge rules, a choice that in tmMCSs can be a limitation. In fact, contexts to be queried are not necessarily fully known in advance; rather, it can become known only at \as run-time'' which are the specific contexts to be queried in the situations that practically arise. For instance, in \cite{CostantiniF16a} we have proposed an  Agent Computational Environment (ACE), that is actually a tmMCS including an agent, to model a personal assistant-agent assisting a prospective college student; the student has to understand to which universities she could send an application (given the subjects of interest, the available budget for applications and then for tuition, and other preferences), where to perform the tests, where to enroll among the universities that have accepted the application. So, universities, test centers etc. have to be retrieved dynamically (as external contexts) and then suitable bridge rules to interact with such contexts must be generated and executed; in our approach, such bridge rules are obtained exactly by instantiating bridge rules patterns. See Section\rif{dynamic} below for a generalization of mMCS and tmMCS to dynamic systems, whose definition in term of composing contexts can change in time.

Thus, below we generalize and formalize for mMCS and tmMCS the enhanced form of bridge rules proposed in~\cite{CosForEMAS16} for logical agents.
In particular, we replace context names in bridge-rule bodies with special terms, called \emph{context designators},
which are intended to denote a specific \emph{kind} of context.
For instance, a literal such as 
\\\centerline{$\mathit{doctorRoss\,{:}\,prescription(disease,P)}$}
 asking family doctor,
specifically Dr.~Ross, for a prescription related to a certain disease, might become
\\\centerline{$\mathit{family\_doctor(d)\,{:}\,prescription(disease,P)}$}
where the doctor to be consulted is not explicitly specified.
Each context designator must be substituted by a context name prior to bridge-rule activation;
in the example, a doctor's name to be associated and substituted to the context designators $\mathit{family\_doctor(d)}$ (which, therefore, acts as a placeholder)
must be dynamically identified; the advantage is, for instance, to be able to flexibly consult
the physician who is on duty at the actual time of the inquiry.

\begin{definition}
Let $C_i$ be a context of a (t)mMCS.
A \emph{context designator} $m(k)$ is a term where $m$ is a fresh function symbol and $k$ a fresh
constant.\footnote{Meaning that $m$ and $k$ belong to the signature of $L_i$, but they do not occur in $kb_i$ or in $br_i$.}

A \emph{bridge-rule pattern} $\phi$  is an expression of the form:
$$
s \ar ({\cal{C}}_1{:}p_1), \ldots, ({\cal{C}}_j{:}p_j), \no ({\cal{C}}_{j+1}{:}p_{j+1}), \ldots, \no ({\cal{C}}_m{:}p_m)
$$
where each ${\cal{C}}_d$ can be either a constant or a context designator.
\end{definition}

New bridge rules can thus be obtained and applied by replacing, in a bridge-rule pattern, context designators via actual contexts names.
So, contexts will now evolve also in the sense that they may increase their set of bridge rules by exploiting bridge-rule patterns:
\begin{definition}
	Given a (t)mMCS, each of the timed composing contexts $C_i$, $1 \leq i \leq {\ell}$,
is defined, at time $0$, as $C_i[0] = \tuple{c_i,L_i,kb_i[0],br_i,brp_i,OP_i,mng_i,\act_i,{\cal U}_i}$,
where $brp_i$ is a set of bridge-rule patterns, and all the other elements are as in Definition\rif{evolvcontexts}.
\end{definition}

\begin{definition}[rule instance]
An \emph{instance} of the bridge-rule pattern $\phi \in brp_i$, for $1 \leq i \leq {\ell}$,
occurring in an (t)mMCS, is a bridge rule $r$ obtained by substituting every context designator occurring in $\phi$
by a context name $c \in \{c_1,\ldots,c_{\ell}\}$.
\end{definition}

The context names to replace a context designator must be established by suitable reasoning in context's knowledge base.
\begin{definition}
Given a (t)mMCS $M$, any composing context $C_i$, for $1 \leq i \leq {\ell}$, and a timed data state $\vec{S}[T]$ of $M$, 
a \emph{valid instance} of a bridge-rule pattern $\phi \in brp_i$ is a bridge rule $\hat{r}$ obtained by substituting every 
context designator $m(k)$ in $\phi$ with a context name $c \in \{c_1,\ldots,c_{\ell}\}$
such that $S_i[T] \models \mathit{subs_i}(m(k),c)$, 
where $\mathit{subs_i}$ is a distinguished predicate that we assume to occur 
in the signature of $L_i$.
\end{definition}
Let $\mathit{vinst}(brp_i)[T]$ denote the set of valid instances of bridge rule patterns that can be obtained at time $T$ (for $C_i$).
The set of bridge rules associated with a context increases in time with the addition of new ones
which are obtained as valid instances of bridge-rule patterns.
\begin{definition}
	Given a tmMCS, each of the timed composing contexts $C_i$, $1 \leq i \leq\ell$, is defined, at time $\T1$, as:
$$
C_i[\T1] = \tuple{c_i,L_i,kb_i[\T1],br_i[\T1],brp_i,OP_i,mng_i,\act_i,{\cal U}_i}
$$
with $br_i[\T1]{=}br_i[T]{\cup}\mathit{vinst}(brp_i)[T]$,  $br_i[0] = br_i$, and all the rest is defined as done in Definition\rif{evolvcontexts}. 
\end{definition}

All the other previously-introduced notions (namely, equilibria, bridge-rule applicability, etc.) remain unchanged.
Notice that instantiation of bridge-rule patterns corresponds to specializing bridge rules
with respect to the context which is deemed more suitable for acquiring some specific information at a certain stage of a context's operation.
This evaluation is performed via the predicate $\mathit{subs_i}(m(k),c)$ that can be defined
so as to take several factors into account, among which, for instance, trust and preferences.

Notice that a solution that might seem alternative though equivalent to ours could be that of defining \as classical'' bridge rules where each rule has a special element in the body acting as a 'guard' to establish if the rule should be actually applied. However, let us reconsider our working examples. Initially the student does not know which are the universities of interest, as this will the result of a reasoning process and of knowledge exchange with other contexts; the student may even not know which universities exist. So, in the alternative solution it would be necessary to define one bridge rule for each university in the world, where the rule has a ‘guard’ to 'enable' the rule in case it should be finally applied as the student concluded to be interested: this is unfeasible if the student does not know all universities in advance, and it is clearly unpractical. Similarly for doctors, who can retire or die, or start a new practice, or be initially unknown to the patient; so, they can hardly be all listed in advance.
The same holds for many other classes of knowledge sources that can be modeled as contexts. Our solution is in many practical cases more practical and compact. It is unprecedented in the literature and adds actual expressive power. 

This enhancement enables us to pursue the direction of \emph{dynamic} (t)mMCS, where contexts can either join or leave the system during its operation.
Such an extension has been in fact advocated since \cite{BrewkaEF11}.
The meaning is that the set of contexts which compose an mMCS may be different at different times. Here, one can see the usefulness of having 
a constant acting as the context name; in fact, bridge-rule definition does not strictly depend upon 
the composition of the mMCS, as it would be by using integer numbers.
Rather, the applicability of a bridge rule depends on the presence in the system of contexts with the names indicated in the bridge-rule definition.
Not only can contexts be added or removed, but they can be substituted by new \as versions'' with the same name.

\section{Static System Assumption and Unique Source Assumption: Dynamic (Timed) mMCSs (dtmMCS) and Multi-Source Option}
\label{dynamic}

As mentioned, in general, a heterogeneous collection of distributed sources will not necessarily remain static in time.
New contexts can be added to the system, or can be removed, or can be momentarily unavailable due to network problems. 
Moreover, a context may be known by the others only via the role(s) it assumes or the 
services which it provides within the system. Although not explicitly specified in the original MCS definition,
context \emph{names} occurring in bridge-rule bodies must represent all the necessary information for 
reaching and querying a context, e.g., names might be~URIs.

It is however useful for a context to be able
to refer to other contexts via their roles, without necessarily being explicitly aware of their names.
Also, a context which joins an MCS will not necessarily make itself
visible to every other contexts: rather, there might be specific authorizations involved.
These aspects may be modeled by means of the following extensions. 

\begin{definition}
A \emph{dynamic} managed (timed) Multi-Context System (d(t)mMCS) at time $T$ is a (t)mMCS augmented with the two special contexts $\D$ and $\R$, which have no associated bridge rules, bridge-rule patterns, and update operator, and where:
\begin{itemize}
\item
$\D$ is a \emph{directory} which contains the list of the contexts, namely $C_1,\ldots,C_{\ell}$,
participating in the system at time~$T$
where, for each $C_i$, its name is associated with its \emph{roles}. We assume $\D$ to admit queries 
of the form '$\mathit{role@Dir}$', returning the name of some context with role '$\mathit{role}$', where '$\mathit{role}$' is 
assumed to be a constant. 
\item
$\R$ contains a directed graph determining which other contexts are \emph{reachable} from each context $C_i$.
For simplicity, we may see $\R$ as composed of couples of the form $(C_r,C_s)$ meaning that context $C_s$ is (directly or indirectly) reachable from
context $C_r$. 
\end{itemize}
\end{definition}

A d(t)mMCS is a more structured system w.r.t. MCS. Via $\D$, contexts can access sources of knowledge without knowing them in advance, via a centralized system facility listing available contexts, with their role. In practice, contexts joining a d(t)mMCS will be allowed to register to $\D$ specifying their role, where this registration can possibly be optional. Via $\R$, the system assumes a structure that may correspond to the nature the applications, where a context can sometimes be not allowed to access all the others, either for a matter of security or for a matter of convenience. Often, information must be obtained via suitable mediation, while access to every information source is not necessarily either allowed or desirable. For instance, a patient's relatives can ask the doctors about the patient's conditions and prognosis, but they are not allowed (unless explicitly authorized by the patient) to access medical documentation directly. Students can see their data and records, but not those of other students.

For now, let us assume that a query $\mathit{role@Dir} = c$ where $c \in \{C_1,\ldots,C_{\ell}\}$, i.e., 
returns a unique result.
The definition of timed data state remains unchanged. 
Bridge rule syntax must instead be extended accordingly:

\begin{definition}
	\label{newbr}
	Given a d(t)mMCS (at time $T$ if timed) $M[T]$, each (non-ground) bridge rule $r$ in the composing contexts has the form:
	$$\begin{array}{l}
	s \ar ({\cal C}_1: p_1), \ldots, ({\cal C}_j : p_j),  \,\no ({\cal C}_{j+1} : p_{j+1}), \ldots,\no ({\cal C}_m : p_m).
	\end{array}
	$$
	\noindent where for $1 \leq k \leq m$ the expression ${\cal C}_{k}$ is either a context name, or an expression $\mathit{role_k@Dir}$.
\end{definition}

Bridge-rule grounding and applicability must also be revised. In fact, for checking bridge rule applicability:
(i) each expression $\mathit{role_k@Dir}$ must be substituted by its result 
and (ii) every context occurring in bridge rule body must be reachable from the context where the bridge rule occurs. 

\begin{definition}
	\label{preground}
	Let $M$ be a d(t)mMCS (at time $T$ if timed) and $\vec{S}$ be a (timed) data state for $M$. 
	Let $r$ be a bridge rule in the form specified in Definition\rif{newbr}.
	The \emph{pre-ground} version $r'$ of $r$ is obtained by substituting each expression $\mathit{role_k@Dir}$ occurring in the body of $r$
	with its result $c_k$ obtained from $\D$. 
\end{definition}

Notice that $r'$ is a bridge rule in \as standard'' form, and that $r$ and $r'$ have the same head,
where their body differ since in $r'$ all context names are specified explicitly.

\begin{definition}
	\label{contapp}
	Let $r'$ be a pre-ground version of a bridge rule $r$ occurring in context $\hat{C}$ of d(t)mMCS $M$ 
with (timed) data state $\vec{S}$.
Let $\rho$ be a ground instance w.r.t.~$\vec{S}$ of~$r'$.
We have now $hd(\rho) \in \mathit{app}(\vec{S})$ if $\rho$ fulfills the conditions for applicability w.r.t.~$\vec{S}$ and, 
	in addition, for each context $\tilde{C}$ occurring in the body of $\rho$ we have that $(\hat{C},\tilde{C}) \in \R$.
\end{definition}

The definition of equilibria is basically unchanged, save the extended bridge-rule applicability. However, suitable update operators (that we do not discuss here) will be defined for both $\D$ and $\R$, to keep both the directory and the reachability graph up-to-date with respect to the actual system state. The question may arise of where such updates might come from. This will in general depend upon the application at hand: the contexts might themselves generate an update when joining/leaving a system, or some kind of monitor (that might be one of the composing contexts, presumably however equipped with reactive, proactive and reasoning capabilities) might take care of such task. Thus, the system as a whole will have a \as policy'' that defines reachability.
Notice, in fact, that $\R$ provides structure to the system in a global way: each context indeed belongs to the sub-MCS of its reachable contexts. This corresponds to the structure of many applications and introduces a notion similar to \as views'' in databases: a context can sometimes be not allowed to access all the others, either for a matter of security or for a matter of convenience. Often, information must be obtained via suitable mediation, while access to every information source is not necessarily allowed or desirable. Reachability can evolve in time according to system's global policies.

\smallskip 

Via $\D$ contexts are categorized, independently of their names, into, e.g., universities, doctors, etc. So, a context can query the ‘right’ ones via their role. 
However, there might sometimes be the case where a specific context is not able to return a required answer,
while another context with the same role instead would. More generally, we may admit a query $\mathit{role@Dir}$ to return not just one, but possibly several results,
representing the set of contexts which, in the given d(t)mMCS, have the specified role.
So, the extension that we propose in what follows can be called a \emph{multi-source option}.
In particular, for d(t)mMCS $M[T]$, composed at time $T$ of contexts $C_1,\ldots,C_{\ell}$, the expression $\mathit{role_k@Dir}$ 
occurring in bridge rule $r \in br_s$ will now denote some nonempty set $SC_k$ $\subseteq$ ($\{C_1,\ldots,C_{\ell}\}\setminus \{C_s\}$),
indicating the contexts with the required role (where $C_s$ is excluded as a context would not in this case intend to query itself).
Technically, there will be now several pre-ground versions of a bridge rule, which differ relative to the contexts occurring in their body.

\begin{definition}
	\label{multsource}
	Let $M$ be a d(t)mMCS (at time $T$) and $\vec{S}$ be a (timed) data state for~$M$. 
	Let $r \in br_s$ be a bridge rule in the form specified in Definition\rif{newbr} occurring in context~$C_s$.
	A \emph{pre-ground} version $r'$ of $r$ is obtained by substituting each expression $\mathit{role_k@Dir}$ occurring in the body of~$r$
	with $c \in SC_k$. 
\end{definition}

Bridge-rule applicability is still as specified in Definition\rif{contapp} and the definition of equilibria is also basically unchanged.

In practice, one may consider to implement the multi-source option 
in bridge-rule run-time application by 
choosing an order for querying the contexts with a certain role
as returned by the directory. The evaluation would
proceed to the next one in case
the answer is not returned within a time-out, or if the answer is under some respect unsatisfactory (according to the management function).  

A further refinement might consist in considering, among the contexts returned by $\mathit{role@Dir}$, only the \emph{preferred} ones. For instance, among medical specialist a patient might prefer the ones who practice nearer to patient's home, or who have the best ratings according to other patients' feedback.

\begin{definition}[preferred source selection]
	\label{prefcrit}
	Given a query $\mathit{role@Dir}$ with result $SC$, a preference criterion ${\cal P}$ returns a (nonempty) ordered subset $SC^{{\cal P}} \subseteq SC$.
\end{definition}

Different preference criteria can be defined according to several factors such as trust, reliability, fast answer, and others.
Approaches to preferences in logic programming might be adapted to the present setting:
cf., among many, \cite{BrewkaEP14} and the references therein, \cite{BienvenuEtAl2010,BrewkaEtAl1010preferencesNMR} and
\cite{CF-jalgor09,CosForWCP2011}). The definition of a context will now be as follows.

\begin{definition}
	A context $C_i$ included in a d(t)mMCS (except for $\D$ and $\R$) is defined (at time $T$), as $C_i[T] = \tuple{c_i,L_i,kb_i[T],br_i[T],brp_i,OP_i,mng_i,\act_i,{\cal U}_i,{\cal P}_i}$ 
where all elements are as defined before for (t)mMCS, and ${\cal P}_i$ is a preference criterion as specified in Definition\rif{prefcrit}.
\end{definition}

Preferences have been exploited in MCS extensions in \cite{EiterW17} and \cite{PontelliMCS18}
in a very different way with respect to what is done here, to cope with relevant issue.
Both the above-mentioned approaches are however orthogonal to ours,
where the preference criterion is associated to each context, and determines
which other contexts to query given a present situation. 
\cite{PontelliMCS18} aims to reconcile the different contexts' preferences 
about some common issue (for instance, mentioning from their first example
where they apparently consider contexts as agents, whether to drink 
red or white wine if going to restaurant together). \cite{EiterW17}
copes with the situation where some constraint is violated in some context 
depending upon other contexts' results communicated via bridge rules.

More precisely, \cite{EiterW17} observes that \as As the contexts of an MCS are typically autonomous and host knowledge bases that are inherited
legacy systems, it may happen that the information exchange leads to unforeseen conclusions and
in particular to inconsistency; to anticipate and handle all such situations at design time is difficult if
not impossible, especially if sufficient details about the knowledge bases are lacking. Inconsistency
of an MCS means that it has no model'', i.e., no equilibrium. To solve the problem, they propose
to define consistency-restoring rules based upon user preferences,
specified in any suitable formalism (though in their examples they adopt 
CP-Nets \cite{BoutilierBDHP04}). This in order to choose among possible repairs, called \as diagnoses'', 
intended as modifications to bridge rules that might restore consistency. 
Such rules might be included into a special context to be added to given MCS.
The goal of \cite{PontelliMCS18} instead is \as to allow each agent (context) to express its preferences
and to provide an acceptable semantics for MCS with preferences''.
To this aim, to define contexts they adopt \as Ranked logics'', where a partial order
among acceptable belief sets is defined.

\section{Logical Omniscience Assumption and Bridge Rules Application Mechanisms}\label{sect:BRgroundApply}

\label{omniscience}

In an implemented mMCS, as remarked in \cite{BarilaroFRT13}, \as ...computing equilibria
and answering queries on top is not a viable solution.'' So, they assume a given 
MCS to admit an equilibrium, and define a query-answering procedure based upon
some syntactic restriction on bridge-rule form, and involving the application and a concept of \as unfolding'' of positive
atoms in bridge-rule bodies w.r.t.\ their definition in the \as destination'' context.
Still, they assume an open system, where every context's contents are visible to others
(save some possible restrictions). We assume instead contexts to be \emph{opaque}, i.e., that contexts' contents are accessible from the outside
only via queries.

Realistically therefore, the grounding of literals in bridge rule bodies w.r.t.\ the present
data state will most presumably be performed at run-time, whenever a bridge rule is actually applied. Such
grounding, and thus the bridge-rule result, can be obtained for instance by \as executing" or \as invoking'' literals in the
body (i.e., querying contexts) left-to-right in Prolog style. 
In practice, we allow bridge rules to have negative literals 
in their body. To this aim, we introduce a syntactic
limitation in the form of non-ground bridge rules very common in logic programming approaches.
Namely, we assume that
\emph{(i) every variable occurring in the head of a non-ground bridge rule  also occurs in some positive literal of its body;
	and (ii) in the body of each rule, positive literals occur (in a left-to-right order) before negative literals}.

So, at run-time variables in a bridge rule will be incrementally and coherently instantiated via results
returned by contexts.
Clearly, asynchronous application of bridge rules determines evolving equilibria.

As mentioned, we believe that bridge-rule application should not necessarily be reactive but rather, according to a context's own logic, other modalities of
application may exist. For instance, a doctor will be consulted not whenever the doctor is available, but only if the patient is in need. Or, an application to some institution is sent not just when all the needed documentation is available, but only if and when the potential applicant should choose to do that. Thus, in our approach the bridge rules that consult the doctor or issue the application remain in general \as passive'', unless they are explicitly triggered by suitable conditions. Moreover, the results returned by a bridge rule are not necessarily processed immediately, but only if and when the receiving context will have the wish and the need to take such results into account. For instance, the outcome of medical analyses will be considered only when a specialist is available to examine them.
So, in our approach we make bridge-rule application proactive (i.e., performed upon specific conditions) and we detach
bridge-rule application and the processing of the management function.

In the rest of this section we formalize the aspects that we have just discussed. Below, when referring to tmMCS we implicitly possibly refer to their timed dynamic version according to the definitions introduced in previous section.
We are now able to formalize the rule instantiation and application mechanism, via a suitable notion of \emph{potential applicability}. 

A slight variation of the definition of bridge-rule grounding is required, as a generalization of Definition\rif{groundbr}.

\begin{definition}\label{groundbrVaried}
	Let $r \in br_i$ be a non-ground bridge rule in a context $C_i$ of a given tmMCS $M$ with (timed) belief state $\vec{S}[T]$. 
	A ground instance $\rho$ of $r$ w.r.t.~$\vec{S}[T]$ is obtained by substituting every variable occurring in $r$ with ground
	terms occurring in~$\vec{S}[T]$.\footnote{Variables may occur in atoms $(c_j{:}p)$ in the body of $r$, in its head $o(s)$, or in both.}
\end{definition}
By a \as ground bridge rule'' we implicitly mean a ground instance of a bridge rule w.r.t.\ a timed data state.
We now redefine bridge-rule applicability, by first introducing proactive activation.
For each context $C_i$, let
\begin{equation}\label{EqDef:HiT}
H_i[T] = \{s | \mbox{\ there exists rule $\rho \in br_i$ with head $hd(\rho)=s$ at time $T$}\}.
\end{equation}

\noindent
We introduce a \emph{timed triggering function},\,
\(\mathit{tr_i[T]}: KB_i \rightarrow 2^{H_i[T]}\),
which specifies (via their heads) which bridge rules are triggered at time $T$, and so, by performing some reasoning over the present knowledge base contents. Since $H_i[T]$ can be in general an infinite set, $\mathit{tr_i[T]}(\cdot)$ may itself return an infinite set. However, in practical cases, it will presumably return a finite set of rules to be applied at time $T$. 
\begin{definition}
	A rule $\rho \in br_i$ is \emph{triggered at time} $T$ iff $hd(\rho) \in tr_i[T](kb_i[T])$.
\end{definition}

Let $grtr_i[T]$ be the set of all ground instances of bridge rules which have been triggered at time $T$,
i.e., $\rho \in grtr_i[T]$ iff $\rho$ is a ground instance w.r.t.~$\vec{S}[T]$ of some non-ground $r$ such that $hd(\rho) \in tr_i[T](kb_i[T])$.

For a tmMCS $M$, data state $\vec{S}[T]$ and a ground bridge rule $\rho$, let $\vec{S}[T] \models body(\rho)$ represent that the rule body holds in belief state~$\vec{S}[T]$.
\begin{definition}
	The set $\mathit{app}(\vec{S}[T])$ relative to ground bridge rules which are applicable in a timed
	data state $\vec{S}[T]$ of a given tmMCS $M[T] = \{C_1[T],\ldots,C_{\ell}[T]\}$ is defined as:
	$$
	\mathit{app}(\vec{S}[T]) = \{hd(\rho)\ |\ \exists T' \leq T \mbox{\ such that\ }\rho \in \mathit{grtr_i}[T']
	\mbox{\ and } \vec{S}[T] \models body(\rho)\}.
	$$
\end{definition}

By the definition of $\mathit{app}(\cdot)$, a (ground instance of) a bridge rule can be triggered at a time $T'$, but can then become
applicable at some later time $T \geq T'$. Thus, any bridge rule which
has been triggered remains in predicate for applicability, which will occur whenever its body is entailed by some future data state.
One alternative solution that may seem equivalent to the proposed one is that of adding a \as guard'' additional positive subgoal in a bridge-rule body, 
to be proved within the context itself.
Such additional literal would have the role of enabling bridge-rule application. 
However, in our proposal a context is \emph{committed} to actually apply bridge rules which have been triggered as soon as 
they will (possibly) be applicable while an additional subgoal should be re-tried at each stage and, if non-trivial reasoning is involved, this may be unnecessary costly.

The definition of timed equilibria remains unchanged, apart from the modified bridge-rule applicability.
However, the added (practical) expressiveness is remarkable as a context in the new formulation is not just the passive recipient of new information,
but it can reason about which bridge rules to potentially apply at each stage.

As mentioned, for practical applications it may be useful to allow the grounding of literals in bridge-rule bodies to be computed at run-time,
whenever a bridge rule is actually applied.
Because of the adoption of a left-to-right Prolog-style execution and in virtue of the assumptions (i)-(ii) seen earlier,
the grounding of a bridge-rule is obtained by \as invoking'' literals in its
body so as to obtain the grounding of positive literals first, to be extended later to negative ones. 

Each positive literal $(c_j : p)$ in the body of a bridge rule may fail (i.e., $c_j$ will return a negative
answer), if none of the instances of $p$ given the partial instantiation computed so far is entailed
by $c_j$'s present data state. Otherwise, the literal succeeds and the other ones
are (possibly) instantiated accordingly. Negative literals  $\no (c_j : p)$ make sense only if $p$ is ground at the time of invocation,
and succeed if $p$ is \emph{not} entailed by $c_j$'s present data state.
In case either some literal fails or a non-ground negative literal is found, the overall bridge rule evaluation fails without returning results.
Otherwise, the evaluation succeeds, and the result can be elaborated by the management function of the \as destination'' context.
For modeling such a run-time procedural behavior, we introduce \emph{potential applicability}.
Since literals in bridge-rule bodies are ordered left to right as they appear
in the rule definition we can talk about first, second, etc., positive/negative literal.

\begin{definition}\label{potapp}
	A ground bridge rule $\rho \in \mathit{grtr_i}[T']$ for some $T'$, of the form~(\ref{f:br}) with $h$ positive literals \(A_1, \ldots, A_h\) in its body
	is \emph{potentially applicable} to grade $k\leq h$ at time $T \geq T'$ iff given a reduced version $\rho'$ of the rule of the form
	\(s \ar A_1, \ldots, A_k\) we have $\vec{S}[T] \models body(\rho')$. 
\end{definition}

Hence, a triggered bridge rule $\rho$ is potentially applicable to grade $k$ at time $T$ if the timed belief state
at time $T$ entails its first $k$ positive literals. Plainly, we have:
\begin{lemma}
	A ground bridge rule $\rho \in \mathit{grtr_i}[T']$ for some $T'$, of the form~(\ref{f:br}) with $h$ positive literals \(A_1, \ldots, A_h\) in its body
	which is potentially applicable to grade $k\leq h$  at time $T$ is also potentially applicable at time $T$ to grade $k'$, $1 \leq k' < k$.
\end{lemma}
\begin{proof}
	This comes from the fact that if $\vec{S}[T] \models body(\rho')$ with $\rho'$ of the form \(s \ar A_1, \ldots, A_k\), then it will hold that $\vec{S}[T] \models body(\rho'')$ with $\rho''$ of the form \(s \ar A_1, \ldots, A_k'\). This because entailing a conjunction implies entailing any sub-conjunction.	\qed
\end{proof}

\begin{proposition}
	A ground bridge rule $\rho$ is \emph{potentially applicable} to grade $k$ at time $T_{k}$ iff for every 
	$k'$ such that $1 \leq k' \leq k$ there exists $T_{k'} \leq T_{k}$ such that $\rho$ is potentially
	applicable to grade $k'$ at time~$T_{k'}$.
\end{proposition}
\begin{proof}
	This is a plain consequence of the above lemma if $T_{k'} = T_{k}$.\qed
\end{proof}

It will most often be the case that a bridge-rule body will become potentially applicable to a higher and a higher grade as time passes. I.e., often for some $k'$ it will be $T_{k'} < T_{k}$.

\begin{proposition}
	Let $\rho \in \mathit{grtr_i}[T']$ be a ground bridge rule, for some $T'$, of the form~(\ref{f:br}) which is potentially applicable to grade $h$ (or simply \as potentially applicable'') at time $T$;
	then, $hd(\rho) \in \mathit{app}(\vec{S}[T])$ iff 
	for every negative literal $\no (c_k{:}p)$ in the body, $h + 1 \leq k \leq m$, $p \not\in S_k[T]$.
\end{proposition}
\begin{proof}
	The result is obtained trivially from the definition of bridge-rule applicability in a data state, here $\vec{S}[T]$, which requires the positive body to be entailed and the negative body not to be entailed.	\qed
\end{proof}

As the contexts composing a tmMCS may be non-monotonic,
a bridge rule $\rho$ can become potentially applicable at some time and remain so
without being applicable, but can become applicable at a later time if the 
atoms occurring in negative literals are no longer entailed by the belief state at that time.

Notice that, in fact, in a practical distributed setting, literals in the body may succeed/fail at different 
times, depending on the various context update times, and upon network delay.
Let us therefore assume that the success of a positive literal $A$ (resp. negative literal $\no A$) is annotated with the time-stamp 
when such success occurs, of the form $A[T]$ (resp. of the form $\no A[T]$).
So, for any bridge rule $\rho \in \mathit{grtr_i}[T']$ 
of the form~(\ref{f:br}) which is triggered at some time and then executed later, by the expression
\[s[T_s] \ar A_1[T_1], \ldots, A_h[T_h],
\no A_{h+1}[T_{h+1}], \ldots, \no A_m[T_m]\]
we mean that the rule has been executed left-to-right where each literal $A[T]/\no A[T]$ has succeeded at time $T$ and the result has been processed (via the management function) by the destination context (i.e., the one where the bridge rule occurs) at time $T_s$.
Clearly, we have $T_1 \leq T_2 \leq ... \leq T_m \leq T_s$. 
The above expression is called the \emph{run-time version} of rule $\rho$. 
It is \emph{successful} if all literals succeed; otherwise it is \emph{failed}. 

In many practical cases we are able to assume a certain degree of persistence in the system,
i.e., that a literal which succeeds at a time $T_r$ would then still succeed (if re-invoked) at every subsequent time 
$T_q$ with $T_r \leq T_q \leq T_s$ (\as local persistence assumption'' for rule $\rho$).
This corresponds to assuming that during the execution of a bridge-rule no changes
in the involved contexts occur that would invalidate the bridge-rule result before actual completion;
so, this would enforce what we call the \as coherent'' execution of a bridge rule.
This assumption may sometimes be problematic in practice, but
in a distributed system such as a tmMCS we cannot realistically assume the execution to be instantaneous.
The assumption of persistence can be considered as reasonable whenever the time amount required by the execution
of a bridge rule is less than the average system change rate with respect to the involved belief elements.
This is the usual condition for artificial intelligence systems to be adequate w.r.t.\ the environment where
they are situated, rather that being \as brittle''. There are possible ways of enforcing the assumption.
For instance, the contexts involved in a bridge rule execution might commit to keep the involved belief
elements untouched until the destination context sends an \as ack'' to signal the completion of rule execution.
This or other strategies to ensure coherent execution are deferred to the implementation of tmMCS.
Under this assumption we have:

\begin{theorem}\label{brexeth}
	Given a successful run-time version of a ground bridge rule $\rho \in \mathit{grtr_i}(\hat{T})$, for some $\hat{T}$, of the form
	\(s[T_s] \ar A_1[T_1], \ldots, A_h[T_h], \no A_{h+1}[T_{h+1}], \ldots,$ $\no A_m[T_m]\).
	Then, $hd(\rho) \in \mathit{app}(\vec{S}[T_s])$ and  there exists $T \leq T_m$ such that $hd(\rho) \in \mathit{app}(\vec{S}[T])$.
	Moreover,
	for each $i\leq j$ there exists $T'_i \leq T_i$ and $T_{i-1} \leq T'_i$ if $i > 1$ such that $\rho$
	is potentially applicable to grade $i$ at time $T'_i$.
\end{theorem}
\begin{proof}
	(Sketch). Thanks to the persistence of the system, putting $T'_i = T_i$, for each $i\leq j$, suffices to prove the first property,
	and putting $T=T_m$ suffices for the second one. The third one follows in a straightforward way. \qed
\end{proof}

The relevance of the above theorem is because, in practice, 
the execution involves a non-ground rule $r \in \mathit{tr_i}[\hat{T}](kb_i[\hat{T}])$ whose rule $\rho$ mentioned therein is a ground instance.
Then, successful left-to-right execution of literals in the body of $r$ will lead to dynamically generate at run-time a successful run-time version of $\rho$.
Specifically, the execution of each non-ground positive literal in the body of $r$ will generate a partial
instantiation which is propagated  to the rest of the bridge-rule body.

\section{Case Study}
\label{casestudy}

In this Section we provide an example of use to illustrate the new features.
We refer to Example 2 in \cite{BrewkaEP14} where it is supposed that a user, Bob, suffering from dementia, is able to live at home thanks to an assisted living system,
modeled by means of an rMCS. The system is in fact able to constantly track Bob's position, and to detect and take care of emergencies (such as a sudden illness, but also Bob's forgetting the stove on, and the like), on the basis of the data provided by sensors suitably installed in Bob's flat. 

We generalize the rMCS to an F\&K which is able to detect changes in Bob's health and is capable to notice anomalous behavior which might signify an impaired state. The inputs to the system can be provided by means of sensors, and by recording and monitoring Bob's activities via suitable telemedicine appliances: by adopting ubiquitous sensing technologies, such appliances since long (cf., e.g., \cite{telemedicine2012}) include several kinds of smart devices such as, e.g., beds that monitor sleep, chairs that monitor breathing and pulse, video cameras that monitor general wellbeing, and more. We assume that the overall F\&K system can attend to several patients. Each patient, among which Bob, is in care of a personalized assistant/monitoring agent (PMA) participating in the F\&K. In \cite{telemedicine2012} it is in fact remarked that, without such an integrated system, sophisticated telemedicine appliances are \as little more than emergency alarm systems''. A PMA might be incarnated, in different use-cases, in a chatterbot/avatar on TV or on the mobile phone, in a smart device, or even in a social/care robot. In our view, a PMA should be able to \as transmigrate'' from one form to another one according to the situation which a patient is experimenting at a certain stage of her/his activities, so that a PMA will never desert \as her'' patient. The patient's PMA (in this case Bob's), is assumed below to be modeled as a context included in an mMCS where the PMA is in particular an agent defined in any agent-oriented logic language~\cite{CostantiniDM1}. As said, this agent will be adequately situated and connected to the patient.

Bob's PMA $pma_{bob}$ might for instance, in case of slight variations of, e.g., blood pressure or blood coagulation, re-adjust the quantity of medicament by consulting a treatment knowledge base. In case of more relevant symptoms, the agent might consult either Bob's physician or a specialist, depending upon the relative importance of symptoms. In case of an emergency, for instance a hemorrhage, immediate help is asked for. Clearly, the example is over-simplified and we do not aim at medical accuracy.

Following \cite{BrewkaEP14}, in the example we assume that
present \as objective'' time is made available to each context as $\mathit{now(T)}$, where the current value of $T$ is entered into the system by a particular sensor.
As previously specified, literals in bridge-rule bodies with no indication of a context are to be proved locally to the context to which the bridge rule belongs.

The following bridge rule associated to each patient's PMA, and thus also to $pma_{bob}$, is maybe the most crucial as it is able to ask for immediate help.
The last literal in the body concerns a context $\mathit{emergency\_center}$, which is supposed to be an interactive component for emergency management able to send, for instance, an ambulance or a helicopter for transportation to the hospital. The last literal in bridge-body will succeed whenever the $\mathit{emergency\_center}$ context has received and processed the request, which is sent in case severe symptoms have been detected.
Those symptoms and the time when the request is issued are communicated to the emergency center together with the patient's data (here, for simplicity, just a patient's identification code). The bridge rule head is processed by the management function so as to simply add to the PMA's knowledge base the record of the fact that help has been required at time $T$ for set of severe symptoms $S$. 

\smallskip\noindent
\(\begin{array}{l}
\mathit{help\_asked(bob,S,T,T1,Th,H)} \ar\\
\tbl \mathit{now(T)}, \mathit{detected\_symptoms(S,T1)}, \mathit{T1} \leq T,\\
\tbl \mathit{patientid(bob,Bid)},~ 
\mathit{emergency\_center : urgent\_help(Bid,S,T,Th,H)}
\end{array}\)

\smallskip
The emergency center is provided with Bob's patient's id $Bid$ and returns the time $Th$ and mean $H$ by which the emergency will be coped with (e.g., an ambulance by 15 minutes).
In order to ensure application of the rule only in case of a real emergency,
it will be triggered (by adding its head to $\mathit{tr_{pma_{bob}}}[T]$) only upon specific conditions,
e.g., if symptoms have occurred that can be considered to be severe for Bob's particular health conditions and a physician is not already present or promptly available. Upon bridge-rule application, the  management function will record the 
(ground) request for help $\mathit{help\_asked(bob,S,T,T1,Th,H)}$ that has been issued; this allows for instance the PMA to check whether the promised assistance will actually arrive in time and possibly to minister palliative treatment in the meanwhile.

The following bridge rule, potentially crucial for cardiopathic patients, will be triggered by in case the blood coagulation value detected at time $T$ is anomalous;
this implies that the quantity of anti-coagulant which Bob takes to treat his heart disease must be rearranged accordingly. The correct quantity $Q$ is obtained by the ATC (Anti-Coagulant Center) knowledge base according to the last blood coagulation value $V$ and its variation $D$ from previous records.

\smallskip\noindent
\(\begin{array}{l}
\mathit{quantity(anticoagulant,Q)} \ar\\
\tbl \mathit{coagulation\_val(V,D)},
\mathit{patientid(bob,Bid)},\ 
\mathit{atc : quantity(Bid,V,D,Q)}
\end{array}\)

\smallskip
In case a patient's health state is not really critical but is anyway altered, a 
physician must be consulted. However, in case, for example, of a simple flu the family doctor suffices, while if there are symptoms that might be related to a more serious condition then a specialist (e.g., a cardiologist) should be consulted. Thus, there will be a bridge-rule pattern of the following form, where a physician can be consulted for condition $C$ (where $C$ is either symptom or a list of symptoms). Again, the management function will record the request having been sent; the last literal in the body will succeed, when the rule is dynamically executed, as soon as the doctor receives the request.

\smallskip\noindent
\(\begin{array}{l}
\mathit{call\_physician(bob,T)} \ar\\
\tbl \mathit{now(T)}, \mathit{condition(bob,T,C)},\\
\tbl \mathit{patientid(bob,Bid)},\ 
\mathit{mydoctor(d) : consultation\_needed(Bid,C,T)}
\end{array}\)

\smallskip
The physician, represented by the context designator $\mathit{mydoctor(d)}$, should however be determined according to $C$. 
This can be done by suitably augmenting the definition of the distinguished predicate $\mathit{subs_{pmi\_bob}}$, for instance, as follows. The notation '$\_$' indicates a \as don't care'' variable, as time is not taken into account here.

\smallskip
\noindent
\(\begin{array}{l}
\mathit{subs_{pmi\_bob}(mydoctor(d),F)} \ar\ 
\mathit{family\_doctor(F)}, \mathit{condition(bob,\_,fever)}\\
\mathit{subs_{pmi\_bob}(mydoctor(d),F)} \ar\ 
\mathit{family\_doctor(F)}, \mathit{condition(bob,\_,headache)}\\
\mathit{subs_{pmi\_bob}(mydoctor(d),G)} \ar\ 
\mathit{my\_cardiologist(G)}, \mathit{condition(bob,\_,chestpain)}\\
\mathit{my\_neurologist(G)}, \mathit{condition(bob,\_,[dizziness,swoon])}\\
\ldots
\end{array}\)

\smallskip
Thus, a valid instance of the bridge-rule pattern will be generated according to the patient's need, as evaluated by the patient's PMI. 
The resulting bridge rule will then be immediately executed. So for instance, if Bob has chest pain and the cardiologist who has been following him is Dr. House, then the bridge rule below will be constructed and triggered:

\smallskip\noindent
\(\begin{array}{l}
\mathit{call\_physician(bob,T)} \ar\\
\tbl \mathit{now(T)}, \mathit{condition(bob,T,chest\_pain)},\\
\tbl \mathit{patientid(bob,Bid)},
\mathit{drHouse : consultation\_needed(Bid,chest\_pain,T)}
\end{array}\)

As a further generalization, the PMA might also try to retrieve a doctor (for instance, a dermatologist, assuming that Bob has never consulted one in recent times, and so there is no record in the PMA's knowledge base) via the system's directory; in this case $\mathit{subs_{pmi\_bob}}$, would include a rule such as:

\smallskip
\noindent
\(\begin{array}{l}
\mathit{subs_{pmi\_bob}(mydoctor(d),F)} \ar\ 
\mathit{dermatologist@Dir}, \mathit{condition(bob,\_,eczema)}
\end{array}\)

Suitable preferences might be specified, stating for instance that a dermatologist who is an expert in allergic problems is preferred (if any can be found), and in case more than one can be located, the doctor who has the medical office nearer to Bob's home is preferred. So we might state:

\smallskip
\noindent
\(\begin{array}{l}
{\cal P}_{pmi\_bob} = \{\mathit{expert(dermatology,allergy), near\_home(BobAddress)}\}
\end{array}\)

Such specification presumes that the Directory is made aware of contexts' preferences and is thus able to perform this kind of evaluation.

\section{Complexity}

We briefly discuss the complexity issues related to the proposed approach.
In general, the property that we may wish to check is whether a specific belief of our interest
will eventually occur at some stage in one (or all) timed equilibria of a given mMCS.
The formal definition is the following.
	
\begin{definition}
The problem $Q^{\exists}$ (respectively $Q^{\forall}$), consists in deciding whether,
for a given mMCS $M$ under a sequence $\Pi = \Pi[1],\Pi[2],\ldots,\Pi[t]$ of update actions performed at time instants $1,2,\ldots,t$,
and for a context $C_i$ of $M$ and a belief $p_i$ for $C_i$, it holds that $p_i \in S_i[t']$ for some
(respectively for all) timed equilibria $\vec{S}[t']$ at time $t' \leq t$.
\end{definition}

We resort, like \cite{BrewkaEP14}, to \emph{context complexity} as introduced in \cite{EiterFSW10}.
One has first to consider a \emph{projected belief state} $\hat{\vec{S}}[t]$,
which includes in the element $\hat{S}_i[t]$ the belief $b_i$ one wants to query,
and also includes for every element $\hat{S}_j[t]$ the beliefs that contribute to
bridge-rule applications which may affect $p_i$
(see \cite{BarilaroFRT13} for an algorithm which computes such sets of beliefs for rMCSs).
Then, the context complexity for $C_i$ is the complexity of establishing whether the element $\hat{S}_i[t]$
of such projected belief state is a subset of the corresponding element $\hat{S}_i[t]$ of some timed equilibrium at time $t$.

\begin{definition}
The system context complexity of $M$ is a (smallest) upper bound for
	the context complexity classes of its component contexts.
\end{definition}

Thus, the system context complexity depends upon the logics of the contexts in $M$. 
The problems $Q^{\exists}$ and $Q^{\forall}$ are however undecidable for infinite update sequences,
because contexts' logics can in general simulate a Turing Machine and thus such problems reduce to the \emph{halting} problem.
Actual complexity results can however be obtained under some restrictions. 
In particular, we assume that all contexts $C_i$'s knowledge bases and belief states are finite at any stage,
all update functions ${\cal U}_{i}$, management functions $mng_i$ and triggering functions $\mathit{tr_i}$ are computable in polynomial time,
and that the set of bridge-rule patterns is empty and all bridge rules are ground. then we can state the following.

\begin{theorem}
For finite update sequences, the system context complexity determines the complexity of membership of $Q^{\exists}$ and, complementarily, the complexity of $Q^{\forall}$. Hardness holds if it holds for the system context complexity.	
\end{theorem}

{\bf Proof} A projected belief state can be guessed for each stage $T \in \mathbb{N}$ by a non-deterministic Turing machine.
Then, the inclusion of each such projected belief state $\hat{\vec{S}}[T]$ in some (all) timed equilibria $\vec{S}[T]$ at that stage
can be established by an oracle under the system's context complexity and, if the answer is positive,
it must be checked whether $p_i \in \hat{S}_i[T]$; if not, subsequent updates must be performed (in polynomial time) over $\hat{\vec{S}}[T]$,
and the two checks must be repeated at each stage; this until either $p_i$ is found or time $T$ is reached,
thus obtaining either a positive or a negative answer to the $Q^{\exists}$ problem.
\meno\meno

The complexity of consistency checking for an mMCS $M$, i.e., the complexity of deciding whether it has
some equilibrium, has been discussed in \cite{BrewkaEF11} and remains the same for many variants of (m)MCS, including our own, as it depends on context complexity. Some cases are shown in tables reported in in \cite{BrewkaEF11,BrewkaEP14,GoncalvesKL14a}.

Reconsidering bridge-rule patterns and non-ground bridge rules we observe that
by assuming that $\mathit{subs_{i}}$ can be computed in polynomial time and produces a finite number of substitutions for each context designator,
and given a set of bridge-rule patterns of a certain cardinality (size) $\hat{c}$,
the size of the set of its valid instances is in general single exponential in $\hat{c}$.
The same holds in principle for the grounding of bridge rules.
However, in our setting we assume that bridge-rule are executed and grounded in a Prolog-like fashion:
therefore, the computational burden for grounding is consequently smaller.

\section{Concluding Remarks}
\label{conclusions}

In this paper we have proposed extensions to MCSs,
aimed at practical application of such systems.
We have outlined a number of aspects that could be improved, and we have then specified and formalized possible improvements. The proposed modifications are incremental (though they might be employed separately), and in fact each definition is presented as an extension of previous ones. Thus, the final result is an overall improved MCS formalization coping with all previously emphasized problematic aspects. Notice that, given a dtmMCS at time $t$, the system enjoys all the properties of the corresponding mMCS, i.e., the mMCS obtained by including the same contexts (except $\D$ and $\R$), each with the associated bridge rules (neglecting bridge rules patterns), and where each expression of the form $\mathit{role@Dir}$ has been substituted by a context name according to $\D$ and $\R$, and to the preferences.

The enhanced customizable bridge rules that we have formalized here have been successfully experimented
(though in a preliminary embryonic version) in a significant case-study, reported in \cite{CostantiniEMAS15}.
We intend to employ the proposed features in the implementation, that will start in the near future, of an F\&K Cyber-Physical System and in the activities of the DIGFORASP COST Action on Digital Forensics.
This project will allow us to perform practical experiments in order to assess the performance of this kind of systems (as in fact, as seen, not much can be said about complexity)
and to identify possible limitations and/or further aspects that can be subjected to improvements.



\end{document}